\documentclass{article}

\usepackage{microtype}
\usepackage{graphicx}
\usepackage{booktabs}
\usepackage{enumitem}
\usepackage{pgfplots}
\usepackage{subcaption}
\usepackage{hyperref}
\usepackage{wrapfig}

\usepackage[accepted]{icml2025}

\usepackage{amsthm}
\usepackage{amsmath}
\usepackage{amssymb}
\usepackage{mathtools}
\usepackage{amsthm}
\usepackage{bm}

\usepackage[capitalize,noabbrev]{cleveref}

\theoremstyle{plain}
\newtheorem{theorem}{Theorem}[section]

\newtheorem{lemma}[theorem]{Lemma}

\theoremstyle{definition}
\newtheorem{definition}[theorem]{Definition}

\theoremstyle{remark}

\usepackage[textsize=tiny]{todonotes}

\newcommand{\norm}[1]{\left\lVert#1\right\rVert}
\newcommand{\abs}[1]{\lvert#1\rvert}

\usepackage{url}
\usepackage{array}
\usepackage{amsfonts}
\usepackage{nicefrac}
\usepackage{microtype}
\usepackage{xcolor}

\icmltitlerunning{On the Importance of Gaussianizing Representations}

\begin{document}

\twocolumn[
\icmltitle{On the Importance of Gaussianizing Representations}

\begin{icmlauthorlist}
\icmlauthor{Daniel Eftekhari}{csc,vec}
\icmlauthor{Vardan Papyan}{mat,csc,vec}
\end{icmlauthorlist}

\icmlaffiliation{csc}{Department of Computer Science, University of Toronto, Toronto, Canada}
\icmlaffiliation{vec}{Vector Institute, Toronto, Canada}
\icmlaffiliation{mat}{Department of Mathematics, University of Toronto, Toronto, Canada}

\icmlcorrespondingauthor{Daniel Eftekhari}{defte@cs.toronto.edu}

\icmlkeywords{mutual information game, gaussianization, power transform, noise robustness, normalization layer, deep learning, information theory}

\vskip 0.3in
]

\printAffiliationsAndNotice{}

\begin{abstract}
\label{abstract}
The normal distribution plays a central role in information theory --
it is at the same time the best-case signal and worst-case noise distribution,
has the greatest representational capacity of any distribution,
and offers an equivalence between uncorrelatedness and independence for joint distributions.
Accounting for the
mean and variance of activations throughout the layers of
deep neural networks
has had a significant effect
on facilitating their effective training,
but seldom has
a
prescription
for
precisely
what
distribution these activations should take,
and how this might be achieved, been
offered.
Motivated by the information-theoretic properties of the normal distribution,
we address
this question
and concurrently present
normality normalization: a novel
normalization
layer
which
encourages normality in the feature representations of
neural networks
using the power transform
and
employs
additive Gaussian noise
during
training.
Our experiments comprehensively demonstrate the effectiveness of normality normalization,
in regards to its generalization performance on an array of widely used model and dataset combinations,
its
strong performance across
various common factors of variation
such as model
width, depth, and training minibatch size,
its suitability
for usage
wherever existing normalization layers are conventionally used,
and as a means to
improving
model robustness to random perturbations.
\end{abstract}

\section{Introduction}
\label{introduction}
\pdfbookmark[1]{Introduction}{bookmark-introduction}
The normal distribution is unique -- information theory shows that among all distributions with the same mean and variance, a signal following this distribution encodes the maximal amount of information \citep{6773024}.
This can be viewed as a desirable property
in learning systems such as neural networks,
where the activations of successive layers equivocates to successive representations of the data.

Moreover,
a signal following the normal distribution is maximally robust to random perturbations \citep{10.5555/1146355}, and thus presents a desirable property for the representations of learning systems; especially deep neural networks, which are susceptible to
random \citep{Ford2019AdversarialEA}
and adversarial \citep{DBLP:journals/corr/SzegedyZSBEGF13} perturbations.
Concomitantly,
the normal distribution is
information-theoretically
the
worst-case perturbative noise distribution
\citep{10.5555/1146355},
which
suggests models
gaining
robustness
to Gaussian noise
should
be robust
to any
other
form of random perturbations.

We show that encouraging deep learning models to encode their activations using the normal distribution
in conjunction with
applying additive Gaussian noise
during training,
helps improve generalization.
We do so by means of a novel layer -- normality normalization -- so-named because
it applies the
power transform, a technique used to gaussianize data \citep{b6d53586-2890-3ac6-bec5-c3cfdcb64879,Yeo2000},
and because it can be viewed as
an
augmentation
of
existing normalization techniques such as batch \citep{pmlr-v37-ioffe15},
layer \citep{DBLP:journals/corr/BaKH16},
instance \citep{Ulyanov2016InstanceNT},
and group \citep{Wu2018ECCV}
normalization.

Our experiments comprehensively demonstrate the general effectiveness of normality normalization,
in terms of
its generalization performance,
its
strong
performance
across
various common factors of variation
such as model
width, depth, and training minibatch size,
which furthermore
serve to
highlight
why
it
is effective,
its suitability
for usage
wherever existing normalization layers are conventionally used,
and its
effect on improving model robustness under random perturbations.

In Section \ref{motivation}
we
outline some of the desirable properties normality can imbue in learning models,
which serve as motivating factors for the development of normality normalization.
In Section \ref{background}
we provide
a brief background on the power transform, before presenting normality normalization in Section \ref{normality-normalization}.
In Section \ref{experiments} we describe
our experiments, analyze the results, and explore some of the properties of models trained with normality normalization.
In Section \ref{related-work}
we comment on related work and discuss some possible future directions.
Finally in Section \ref{conclusion} we contextualize normality normalization in the broader deep learning literature, and provide a few concluding remarks.

\pdfbookmark[1]{Motivation}{bookmark-motivation}
\section{Motivation}
\label{motivation}
In this section we present
motivating factors
for
encouraging normality in feature representations in conjunction with using additive random noise during learning.
Section \ref{experiments}
substantiates
the applicability of the motivation through the experimental results.

\pdfbookmark[2]{Mutual Information Game \& Noise Robustness}{bookmark-mutual-information-game}
\subsection{Mutual Information Game \& Noise Robustness}
\label{mutual-information-game}

\pdfbookmark[3]{Overview of the Framework}{bookmark-mutual-information-game-overview}
\subsubsection
{Overview of the Framework}
Under first and second moment constraints,
the normal distribution is at the same time the best possible signal distribution, and the worst possible noise distribution; a result which can be studied in the context of the Gaussian channel \citep{6773024}, and through the lens of the mutual information game \citep{10.5555/1146355}.
In this framework,
$X$ and $Z$ denote
two independent random variables,
representing
the
input
signal and noise, and
$Y = X + Z$
is the output.
The mutual information between $X$ and $Y$ is denoted by $I\left(X;Y\right)$; $X$ tries to maximize this term, while $Z$ tries to minimize it. Both $X$ and $Z$ can encode their signal using any probability distribution, so that their respective objectives are optimized
for.

Information theory answers the question of what distribution $X$ should choose to maximize $I\left(X;Y\right)$. It also answers the question of what distribution $Z$ should choose to minimize $I\left(X;Y\right)$.
As shown by the following theorem,
remarkably
the answer to both questions is the same -- the normal distribution.

\begin{theorem}\citep{10.5555/1146355}
Mutual Information Game.
Let $X$, $Z$ be independent, continuous random variables with non-zero support over the entire real line, and satisfying the moment conditions $E\left[X\right]=\mu_{x}$, $E\left[X^{2}\right]=\mu_{x}^{2}+\sigma_{x}^{2}$ and $E\left[Z\right]=\mu_{z}$, $E\left[Z^{2}\right]=\mu_{z}^{2}+\sigma_{z}^{2}$. Further let $X^{*}$, $Z^{*}$ be normally distributed random variables satisfying the same moment conditions, respectively. Then the following series of inequalities holds
\begin{equation}
\begin{aligned}
	&I\left(X\phantom{*};X\phantom{*}+Z^{*}\right)\le\\
	&I\left(X^{*};X^{*}+Z^{*}\right)\le\\
	&I\left(X^{*};X^{*}+Z\phantom{*}\right).
\end{aligned}
\label{mig-inequalities}
\end{equation}
\end{theorem}
\begin{proof}
Without loss of generality let $\mu_{x}=0$ and $\mu_{z}=0$. The first inequality hinges on the entropy power inequality. The second inequality hinges on the maximum entropy of the normal distribution given first and second moment constraints. See \citet{10.5555/1146355} for details.
\label{mig-theorem}
\end{proof}
This leads to the following minimax formulation of the game
\begin{equation}
\begin{aligned}
	\min\limits_{Z}\max\limits_{X}I\left(X;X+Z\right)
	= \max\limits_{X}\min\limits_{Z}I\left(X;X+Z\right),
\end{aligned}
\label{mig-minimax}
\end{equation}
which implies that any deviation from normality, for $X$ or $Z$, is suboptimal from that player's perspective.

\pdfbookmark[3]{Relation to Learning}{bookmark-mutual-information-game-learning}
\subsubsection{Relation to Learning}
How might this framework relate to the learning setting?
First, previous works have shown
that
adding noise to the inputs \citep{6796505} or to the intermediate activations \citep{Srivastava2014DropoutAS} of neural networks can be an effective form of regularization, leading to better generalization.
Moreover,
the mutual information game
shows that,
among encoding distributions
with first and second moment constraints,\footnote{Note that conventional normalization layers respect precisely these constraints, since the mean and variance are used to normalize the pre-activations.}
the normal distribution is maximally robust to random perturbations.
Taken together these suggest that
encoding activations using the normal distribution is the most effective way of using noise as a regularizer, because a greater degree of regularizing noise in the activations can be tolerated
for the same level of corruption.

Second, the mutual information game suggests gaining robustness to Gaussian noise
is optimal because
it is the worst-case noise distribution.
This suggests adding Gaussian noise -- specifically -- to
activations
during training
should have the strongest regularizing effect. Moreover, gaining robustness to noise has previously been demonstrated
to
imply
better generalization \citep{Arora2018StrongerGB}.

Finally, there exists a close correspondence between the mutual information between the input and the output of a channel subject to additive Gaussian noise, and the minimum mean-squared error (MMSE) in estimating
the input given the output \citep{1412024}. This suggests that when Gaussian noise is added to a given layer's activations, quantifying the attenuation of the noise across the subsequent layers of the network, as measured by the
mean-squared error (MSE)
relative to the unperturbed activations, provides
a
measurable proxy for the mutual information between the activations of successive layers
in the presence of noise.

\pdfbookmark[2]{Maximal Representation Capacity and Maximally Compact Representations}{bookmark-maximally-compact-representations}
\subsection{Maximal Representation Capacity and Maximally Compact Representations}
\label{maximally-compact-representations}
The entropy of a random variable is a measure of
the number of bits it can encode \citep{6773024},
and therefore
of its representational capacity \citep{10.5555/1146355}.
The normal distribution is the maximum entropy distribution for specified mean and variance. This
suggests
that a unit which encodes features using the normal distribution has maximal representation capacity given a fixed variance budget, and therefore encodes information as compactly as possible.
This may then suggest that it is efficient for a unit (and by extension layer) to encode its activations using the normal distribution.

\pdfbookmark[2]{Maximally Independent Representations}{bookmark-maximally-independent}
\subsection{Maximally Independent Representations}
\label{text-maximally-independent}
Previous work
has explored the beneficial effects of decorrelating features
in
neural networks \citep{Huang2018DecorrelatedBN,Huang2019IterativeNB,Pan2019ICCV}.
Furthermore,
other works have shown
that
preventing
feature
co-adaptation
is
beneficial for training deep neural networks \citep{Hinton2012ImprovingNN}.

For any set of random variables, for example representing the pre-activation values of various units in a neural network layer,
uncorrelatedness
does not imply independence in general. But for random variables whose marginals are normally distributed,
then as shown by Lemma
\ref{lemma-minimal-mi-given-marginals},
uncorrelatedness does imply independence when they are furthermore jointly normally distributed.
Furthermore, for any given (in general, non-zero) degree of correlation between the random variables,
they are maximally independent -- relative to any other possible joint distribution -- when they are jointly normally distributed.

We use these results to motivate the following argument:
for a given level of correlation,
encouraging normality in the feature representations of units
would lead to
the desirable property of
maximal independence between them;
in the setting where increased
unit-wise normality also lends itself to increased joint normality.

\pdfbookmark[1]{Background: Power Transform}{bookmark-background}
\section{Background: Power Transform}
\label{background}
Before introducing normality normalization, we
briefly
outline the power transform \citep{Yeo2000}
which
our proposed normalization layer
employs.
Appendix \ref{appendix-nll-derivation} provides the
complete
derivation of the negative log-likelihood (NLL) objective function presented below.

Consider a random variable $H$
from which a sample
$\bm{h} = \left\{h_{i}\right\}_{i=1}^{N}$ is obtained.\footnote{In the context of normalization layers, $N$ represents the number of samples being normalized;
for example in batch normalization, $N = BHW$ for convolutional layers, where $B$ is the minibatch size, and $H,W$ are respectively the height and width of the activation.}
The power transform gaussianizes $\bm{h}$ by applying the
following
function
for each $h_{i}$:
\begin{equation}
\begin{aligned}
	\psi\left(h; \lambda\right) &=
	\begin{cases}
	\frac{1}{\lambda}\left(\left(1+h\right)^{\lambda} - 1\right), & h \ge 0, \lambda \ne 0\\
	\log{\left(1+h\right)}, & h \ge 0, \lambda = 0\\
	\frac{-1}{2 - \lambda}\left(\left(1-h\right)^{2 - \lambda} - 1\right), & h < 0, \lambda \ne 2\\
	-\log{\left(1-h\right)}, & h < 0, \lambda = 2
	\end{cases}.
\end{aligned}
\label{transform}
\end{equation}
The parameter $\lambda$ is obtained using maximum likelihood estimation,
so that the transformed variable is
as normally distributed as possible,
by minimizing the following
NLL:\footnote{To simplify the presentation,
we momentarily defer the cases \(\lambda=0\) and \(\lambda=2\),
and outline the NLL for \(h \ge 0\) only, as the case for \(h < 0\) follows closely by symmetry.
}
\begin{equation}
	\begin{aligned}
		\mathcal{L}\left(\bm{h};\lambda\right)
		&=\frac{1}{2}\left(\log\left(2\pi\right)+1\right) + \frac{1}{2}\log\left(\hat{\sigma}^2\left(\lambda\right)\right)\\
		&- \frac{\lambda-1}{N} \sum_{i=1}^{N}\log\left(1+h_{i}\right),
	\end{aligned}
\label{pll_yj_x>0}
\end{equation}
where
$\hat{\mu}\left(\lambda\right)=\frac{1}{N}\sum_{i=1}^{N}\psi\left(h_{i}; \lambda\right)$ and $\hat{\sigma}^{2}\left(\lambda\right)=\frac{1}{N}\sum_{i=1}^{N}\left(\psi\left(h_{i}; \lambda\right)-\hat{\mu}\left(\lambda\right)\right)^{2}$.

\pdfbookmark[1]{Normality Normalization}{bookmark-normality-normalization}
\section{Normality Normalization}
\label{normality-normalization}
To gaussianize a unit's pre-activations
$\bm{h}$,
normality normalization
estimates \(\hat{\lambda}\)
using the method we
present
in Subsection \ref{lambda-estimate},
and
then
applies the power transform given by Equation \ref{transform}.
It
subsequently
adds Gaussian noise with scaling
as described in
Subsection \ref{perturbative-noise}.
These
steps
are
done between
the normalization and affine transformation steps conventionally performed in other normalization layers.

\pdfbookmark[2]{Estimate of \texorpdfstring{$\hat{\lambda}$}{λ̂}}{bookmark-lambda-estimate}
\subsection{Estimate of $\hat{\lambda}$}
\label{lambda-estimate}
Differentiating Equation \ref{pll_yj_x>0} w.r.t. \(\lambda\)
and
setting the resulting expression to \(0\) does not lead to a closed-form solution for \(\hat{\lambda}\), which suggests an iterative
method for its estimation; for example gradient descent, or a root-finding algorithm \citep{Brent1971AnAW}.
However,
motivated by
the NLL's convexity in \(\lambda\) \citep{Yeo2000}, we use a
quadratic
series
expansion
for its approximation,
which we outline in Appendix \ref{series-expansion-loss}.

With the quadratic form of the NLL,
we can estimate \(\hat{\lambda}\) with one step of the Newton-Raphson method:
\begin{equation}
\begin{aligned}
	\hat{\lambda} &= 1 -
	\frac{\mathcal{L}'\!\left(\bm{h};\lambda=1\right)}{\mathcal{L}''\!\left(\bm{h};\lambda=1\right)},
\end{aligned}
\label{newton-raphson-update}
\end{equation}
where the series expansion has been taken around\footnote{The previously deferred cases of $\lambda=0$ and $\lambda=2$ are thus inconsequential, in the context of computing an estimate $\hat{\lambda}$, by continuity of the quadratic form of the series expansion for the NLL. However, these two cases still need to be considered when applying the transformation function itself.}
\(\lambda_{0}=1\).
The expressions for
\(\mathcal{L}'\!\left(\bm{h};\lambda=1\right)\) and \(\mathcal{L}''\!\left(\bm{h};\lambda=1\right)\) are outlined in Appendix \ref{series-expansion-loss}.

Appendix \ref{estimation-lambda} provides
empirical evidence
substantiating
the similarity between the NLL and its second-order series expansion around \(\lambda_{0}=1\),
and
furthermore demonstrates the accuracy
of obtaining the estimates \(\hat{\lambda}\) using one step of the Newton-Raphson method.

Subsequent
to estimating \(\hat{\lambda}\),
the power transform is applied to each of the pre-activations
to obtain
$x_{i} = \psi\left(h_{i};\hat{\lambda}\right)$.

We next discuss a few facets of the method.
\paragraph{Justification for the Second Order Method}The justification for using the Newton-Raphson method for computing $\hat{\lambda}$ is as follows:
\begin{itemize}[itemsep=4pt, parsep=0pt, topsep=0pt, left=8pt]
\item A first-order gradient-based method would require iterative refinements to its estimates of $\hat{\lambda}$ in order to find the minima,
which would significantly affect
runtime. In contrast, the Newton-Raphson method is guaranteed to find the minima
of the quadratic loss
in one step.
\item A first-order gradient-based method for computing $\hat{\lambda}$ would require an additional hyperparameter for the step size.
Due to the quadratic nature of the
loss,
the Newton-Raphson method necessarily does not require any such additional hyperparameter.
\item The minibatch statistics $\hat{\mu}$ and $\hat{\sigma}^{2}$ are available in closed-form. It is
therefore
natural to seek a closed-form expression for $\hat{\lambda}$, which is
facilitated
by using the Newton-Raphson method.
\end{itemize}

\paragraph{Location of Series Expansion}The choice of taking the series expansion around \(\lambda_{0}=1\) is justified using the following two complementary factors:
\begin{itemize}[itemsep=4pt, parsep=0pt, topsep=0pt, left=8pt]
\item \(\hat{\lambda}=1\) corresponds to the identity transformation, and hence having \(\lambda_{0}=1\) as the point
where
the series expansion is taken, facilitates its recovery if this is optimal.
\item It equivocates to assuming the least about the nature of the deviations from normality in the sample statistics, since it avoids biasing the form of the series expansion for the loss towards solutions favoring \(\hat{\lambda} < 1\) or \(\hat{\lambda} > 1\).
\end{itemize}

\paragraph{Order of Normalization and Power Transform Steps}
Applying the power transform
after
the
normalization
step
is beneficial,
because having zero mean and unit variance activations
simplifies several terms
in the computation of
\(\hat{\lambda}\),
as shown in Appendix \ref{series-expansion-loss},
and improves
numerical stability.

\paragraph{No Additional Learned Parameters}Despite having increased normality in the features, this came at no additional cost in terms of the number of learnable parameters relative to
existing normalization
techniques.

\paragraph{Test Time}
In the case where normality normalization is used to augment batch normalization,
in addition to computing global estimates for $\mu$ and $\sigma^{2}$,
we additionally compute a
global estimate
for \(\lambda\).
These are obtained using the respective training set running averages for these terms,
analogously
with
batch normalization.
At test time, these global estimates $\mu, \sigma^{2}, \lambda$ are used, rather than the test minibatch statistics
themselves.

\pdfbookmark[2]{Additive Gaussian Noise with Scaling}{bookmark-perturbative-noise}
\subsection{Additive Gaussian Noise with Scaling}
\label{perturbative-noise}
Normality normalization applies regularizing additive random noise to the output of the power transform; a step
which is
also
motivated
through the information-theoretic principles
described in
Subsection \ref{mutual-information-game},
and
whose
regularizing
effect is magnified by
having
gaussianized pre-activations.

For
each input indexed by $i \in \left\{1, \ldots, N\right\}$, during training\footnote{
We do not apply additive random noise with scaling at test time.
} we have $y_{i} = x_{i} + z_{i} \cdot \xi \cdot s$, where $x_{i}$ is the $i$-th input's post-power transform value, $z_{i} \sim \mathcal{N}\left(0, 1\right)$, $\xi \ge 0$ is the noise factor, and $s = \frac{1}{N}\norm{\bm{x} - \bar{\bm{x}}}_{1}$
represents the
zero-centered norm of the post-power transform values,
normalized
by
the sample
size
$N$.

Importantly,
scaling each of the sampled noise values $z_{i}$ for a given channel's minibatch\footnote{For clarity the present
discussion assumes the case where normality normalization is used to augment batch normalization. However, the discussion applies
equally
to other normalization layers,
such as layer, instance, and group normalization.}
by
the
channel-specific scaling factor $s$,
leads to
an appropriate
degree of additive noise
for
each of the
channel's
constituent terms $x_{i}$.
This is significant because
for a given minibatch,
each channel's norm will
differ from the norms of other
channels.

Furthermore,
we treat $s$ as a constant, so that its constituent terms are not incorporated during backpropagation.\footnote{Implementationally, this is done by disabling gradient tracking when computing these terms.}
This is significant because the purpose of $s$ is
to scale the additive random noise by the minibatch's
statistics, and not for it to contribute to learning directly by affecting the gradients of
the
constituent terms.

Note that we employ the \(\ell_{1}\)-norm for \(\bm{x}\) rather than the \(\ell_{2}\)-norm
because it lends itself to a more robust measure of dispersion \citep{PHAMGIA2001921}.\\

Algorithm \ref{algorithm1} provides a summary of normality normalization.

\begin{algorithm}[h!]
    \textbf{Input: }{$\bm{u} = \left\{u_{i}\right\}_{i=1}^{N}$}
    \vspace{0.2em}\\
    \textbf{Output: }{$\bm{v} = \left\{v_{i}\right\}_{i=1}^{N}$}
    \vspace{0.2em}\\
    \textbf{Learnable Parameters: }$\gamma, \beta$
    \vspace{0.2em}\\
    \textbf{Noise Factor: }$\xi \ge 0$
    \vspace{0.8em}\\
    \textbf{Normalization}:
    \vspace{0.25em}\\
    $\hat{\mu} = \frac{1}{N}\sum_{i=1}^{N}u_{i}$
    \vspace{0.25em}\\
    $\hat{\sigma}^{2} = \frac{1}{N}\sum_{i=1}^{N}\left(u_{i}-\hat{\mu}\right)^{2}$
    \vspace{0.25em}\\
    $h_{i} = \frac{u_{i}-\hat{\mu}}{\sqrt{\hat{\sigma}^{2}+\epsilon}}$
    \vspace{0.8em}\\
    \textbf{Power Transform and\\
    Scaled Additive Noise}:
    \vspace{0.35em}\\
    $\hat{\lambda} = 1 - \frac{\mathcal{L}'\left(\bm{h};\lambda=1\right)}{\mathcal{L}''\left(\bm{h};\lambda=1\right)}$
	\vspace{0.25em}\\
    $x_{i} = \psi\left(h_{i};\hat{\lambda}\right)$
    \vspace{0.25em}\\
    \vspace{0.25em}
	with gradient tracking disabled:
	\vspace{0.25em}\\
	\phantom{0000}$\bar{x} = \frac{1}{N}\sum_{i=1}^{N}x_{i}$
	\vspace{0.25em}\\
    \phantom{0000}$s = \frac{1}{N}\sum_{i=1}^{N}\abs{x_{i} - \bar{x}}$
    \vspace{0.45em}\\
    sample $z_{i} \sim \mathcal{N}\left(0,1\right)$
    \vspace{0.25em}\\
    $y_{i} = x_{i} + z_{i} \cdot \xi \cdot s$
    \vspace{0.8em}\\
    \textbf{Affine Transform}:
    \vspace{0.25em}\\
    $v_{i} = \gamma \cdot y_{i} + \beta$
	\caption{Normality Normalization}
	\label{algorithm1}
\end{algorithm}

\pdfbookmark[1]{Experimental Results \& Analysis}{bookmark-experiments}
\section{Experimental Results \& Analysis}
\label{experiments}

\pdfbookmark[2]{Experimental Setup}{bookmark-training-details}
\subsection{Experimental Setup}
\label{training-details}
For each model and dataset combination,
\(M=6\) models were
trained,
each with differing random initializations for the model parameters.
Wherever a result is reported numerically, it is obtained using the mean performance and one standard error from the mean
across the $M$ runs.
The best performing models for a given dataset and model combination
are shown in bold.
Wherever a result is shown graphically,
it is
displayed
using the mean performance,
and
its
$95$\% confidence interval
when
applicable.
The training configurations of the models\footnote{Code is made
available at
\url{https://github.com/DanielEftekhari/normality-normalization}.}
are outlined in Appendix \ref{appendix-training-details}.

\pdfbookmark[2]{Generalization Performance}{bookmark-performance-evaluation}
\subsection{Generalization Performance}
\label{performance-evaluation}
We evaluate
layer normality normalization (LayerNormalNorm)
and
layer normalization (LayerNorm)
on a variety of models and datasets,
as shown in Table \ref{results-lnn}.
A similar evaluation is done for batch normality normalization (BatchNormalNorm) and batch normalization (BatchNorm), shown in Table \ref{results-bnn}.

\paragraph{Normality Normalization is Performant}
LayerNormalNorm
generally
outperforms LayerNorm
across multiple architectures
and
datasets,
with a similar trend holding between BatchNormalNorm and BatchNorm.

\paragraph{Effective With and Without Data Augmentations}
Normality normalization is effective for models trained with
(Table \ref{results-lnn})
and without
(Table \ref{results-bnn})
data augmentations.
This is of value
in
application areas such as
time series analysis and
fine-grained
medical image
analysis,
where
it is often not clear what data augmentations are appropriate.

\begin{table}[h!]
\caption{Validation accuracy
across
several datasets
for a vision transformer (ViT) architecture (see training details for model specification),
when using LayerNormalNorm (LNN) vs. LayerNorm (LN).
Data augmentations were employed during training. 
}
\begin{center}
\begin{small}
\begin{sc}
\begin{tabular}{ccc}
\toprule
Dataset & LN & LNN \\
\midrule
SVHN & 94.61 \(\pm\) 0.31 & \textbf{95.78 \(\pm\) 0.21}\\
CIFAR10 & 89.97 \(\pm\) 0.16 & \textbf{91.18 \(\pm\) 0.13}\\
\hline
CIFAR100 & 66.40 \(\pm\) 0.42 & \textbf{70.12 \(\pm\) 0.22}\\
\hline
Food101 & 73.25 \(\pm\) 0.19 & \textbf{79.11 \(\pm\) 0.09}\\
\hline
ImageNet Top1 & 71.54 \(\pm\) 0.16 & \textbf{75.25 \(\pm\) 0.07}\\
\hline
ImageNet Top5 & 89.40 \(\pm\) 0.11 & \textbf{92.23 \(\pm\) 0.04}\\
\hline
\bottomrule
\end{tabular}
\end{sc}
\end{small}
\end{center}
\label{results-lnn}
\end{table}

\begin{table}[h!]
\caption{Validation accuracy for several ResNet (RN)
architecture and dataset combinations, when using BatchNormalNorm (BNN) vs. BatchNorm (BN).
No data augmentations were employed during training.
}
\begin{center}
\begin{small}
\begin{sc}
\begin{tabular}{cccc}
\toprule
Dataset & Model & BN & BNN \\
\midrule
CIFAR10 & RN18 & 88.89 \(\pm\) 0.07 & \textbf{90.41 \(\pm\) 0.09}\\
\hline
CIFAR100 & RN18 & 62.02 \(\pm\) 0.17 & \textbf{65.82 \(\pm\) 0.11}\\
\hline
STL10 & RN34 & 58.82 \(\pm\) 0.52 & \textbf{63.86 \(\pm\) 0.45}\\
\hline
TinyIN Top1 & RN34 & 58.22 \(\pm\) 0.12 & \textbf{60.57 \(\pm\) 0.14}\\
\hline
TinyIN Top5 & RN34 & 81.74 \(\pm\) 0.16 & \textbf{83.31 \(\pm\) 0.13}\\
\hline
Caltech101 & RN50 & 72.60 \(\pm\) 0.35 & \textbf{74.71 \(\pm\) 0.51}\\
\hline
Food101 & RN50 & 61.15 \(\pm\) 0.44 & \textbf{63.51 \(\pm\) 0.33}\\
\hline
\bottomrule
\end{tabular}
\end{sc}
\end{small}
\end{center}
\label{results-bnn}
\end{table}

\begin{figure}[h!]
  \centering
  	  \includegraphics[width=0.80\linewidth]{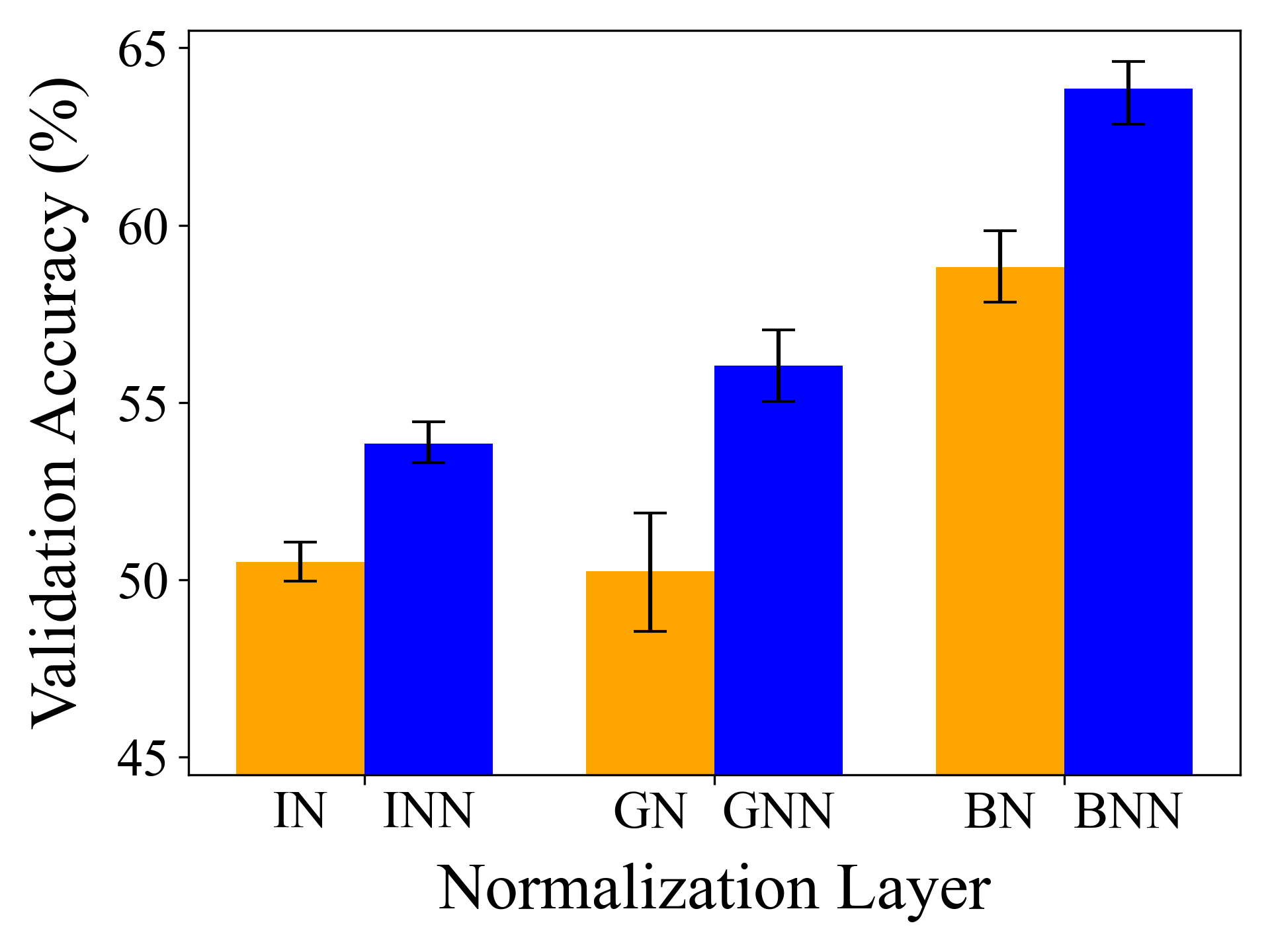}
 	\vspace{-3.5mm}
    \caption{\textbf{Normality normalization is effective for various normalization layers.}
	Validation accuracy for ResNet34 architectures evaluated on the STL10 dataset.
	Each bar represents the performance of the ResNet34 architecture, when using the given normalization layer across the entire network. INN: InstanceNormalNorm, IN: InstanceNorm, GNN: GroupNormalNorm, GN: GroupNorm, BNN:  BatchNormalNorm, BN: BatchNorm.
    }
    \label{figure-other-norm-layers}
\end{figure}

\pdfbookmark[2]{Effectiveness Across Normalization Layers}{bookmark-text-other-norm-layers}
\subsection{Effectiveness Across Normalization Layers}
\label{text-other-norm-layers}
Figure \ref{figure-other-norm-layers} demonstrates the
effectiveness of normality normalization
across various
normalization layer types.
Here we further
augmented
group normalization (GroupNorm) to group normality normalization (GroupNormalNorm), and instance normalization (InstanceNorm) to instance normality normalization (InstanceNormalNorm).

Table \ref{results-dbnn}
furthermore
contrasts decorrelated
batch normalization \citep{Huang2018DecorrelatedBN}
with
its augmented form
decorrelated
batch normality normalization,
providing further evidence
that normality normalization can be employed wherever normalization layers are conventionally used.

\begin{table}[h!]
\caption{As in Table \ref{results-bnn}, but for models using
decorrelated BatchNormalNorm (DBNN) vs. decorrelated BatchNorm (DBN).
}
\begin{center}
\begin{small}
\begin{sc}
\begin{tabular}{cccc}
\toprule
Dataset & Model & DBN & DBNN \\
\midrule
CIFAR10 & RN18 & 90.66 \(\pm\) 0.05 & \textbf{91.50 \(\pm\) 0.03}\\
\hline
CIFAR100 & RN18 & 65.11 \(\pm\) 0.06 & \textbf{67.53 \(\pm\) 0.10}\\
\hline
STL10 & RN34 & 66.76 \(\pm\) 0.29 & \textbf{69.36 \(\pm\) 0.14}\\
\hline
\bottomrule
\end{tabular}
\end{sc}
\end{small}
\end{center}
\label{results-dbnn}
\end{table}

\pdfbookmark[2]{Effectiveness Across Model Configurations}{bookmark-text-model-width}
\subsection{Effectiveness Across Model Configurations}
\label{text-model-width}
\paragraph{Network Width}
Figure \ref{figure-model-width}
shows
that BatchNormalNorm outperforms BatchNorm across varying WideResNet architecture model widths.
Of particular note
is that BatchNormalNorm shows strong performance even in the regime of relatively small network widths, whereas BatchNorm's performance deteriorates.
This may indicate that for small-width networks, which do not exhibit the Gaussian process limiting approximation attributed to large-width networks \citep{10.5555/525544,lee2018deep,NEURIPS20185a4be1fa,NEURIPS20190d1a9651},
normality normalization provides a correcting effect.
This could, for example, be beneficial
for hardware-limited deep learning applications.

\begin{figure}[h!]
  \centering
  	  \includegraphics[width=0.80\linewidth]{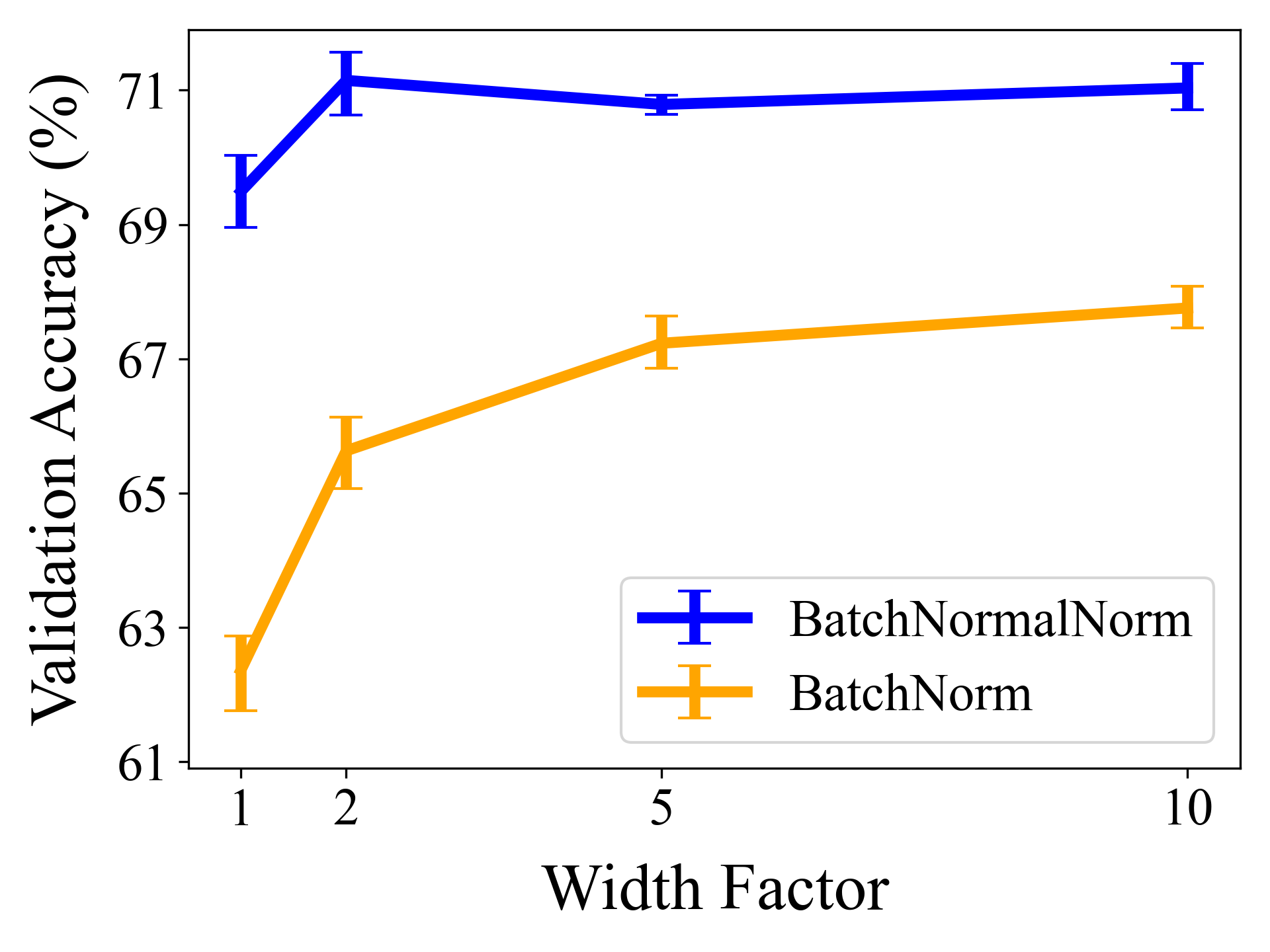}
 	\vspace{-3.5mm}
    \caption{\textbf{Normality normalization is effective for small and large width networks.}
    Validation accuracy on the STL-10 dataset for WideResNet architectures with varying width factors
    when controlling for depth of $28$,
    when using BatchNormalNorm
    vs. BatchNorm.}
    \label{figure-model-width}
\end{figure}

\paragraph{Network Depth}
Figure \ref{figure-model-depth}
shows
that BatchNormalNorm outperforms BatchNorm across varying model depths.
This suggests normality normalization
is beneficial
both for small and large-depth models.
Furthermore,
the increased benefit to performance
for
BatchNormalNorm in deeper networks suggests normality normalization
may
correct
for an increased tendency towards
non-normality as a function of model depth.

\begin{figure}[h!]
  \centering
  	  \includegraphics[width=0.80\linewidth]{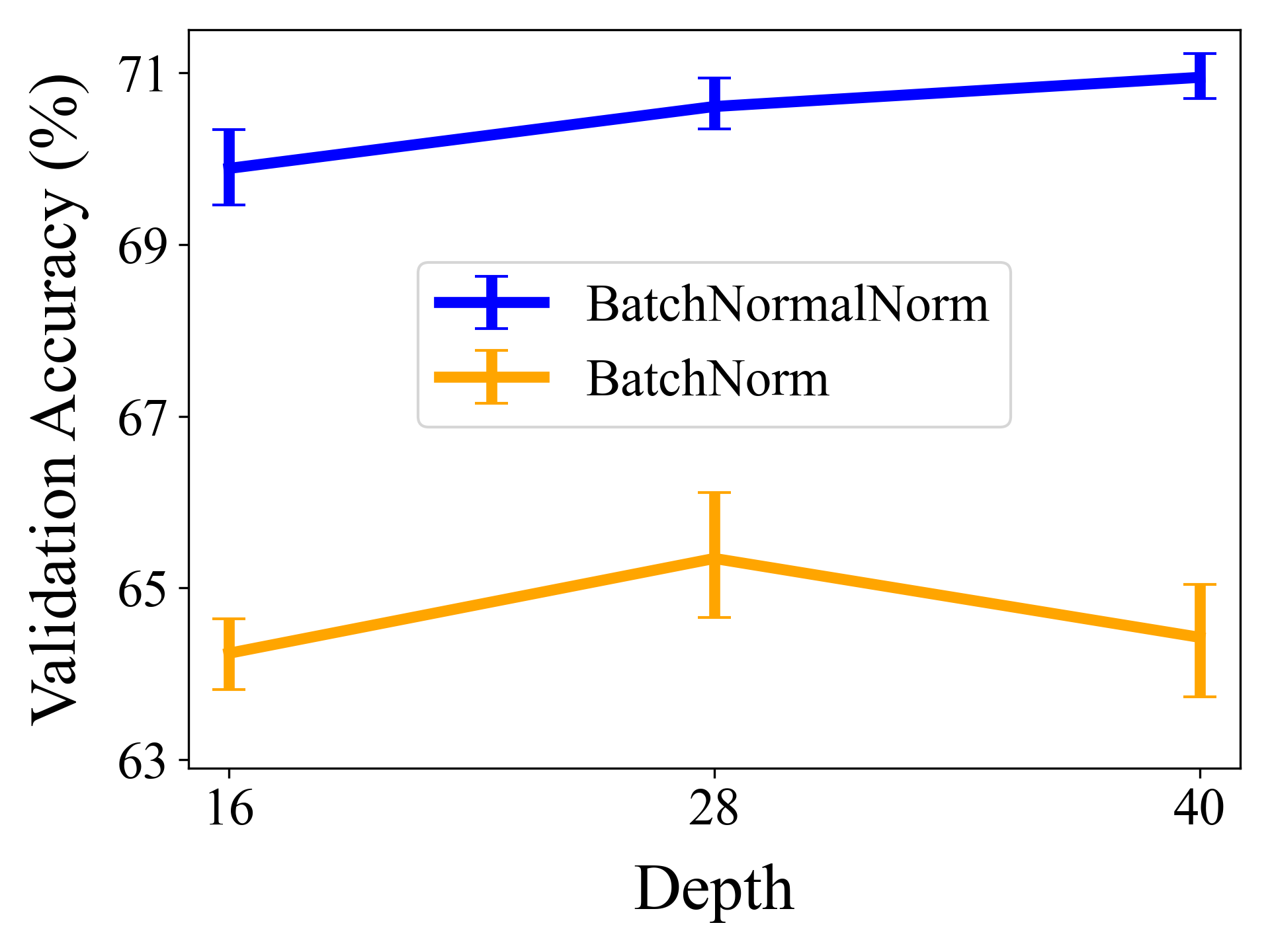}
 	\vspace{-3.5mm}
    \caption{\textbf{Normality normalization is effective for networks of various depths.}
    Validation accuracy on the STL10 dataset for
    WideResNet architectures with varying depths when controlling for a width factor of $2$,
    when using BatchNormalNorm vs. BatchNorm.}
    \label{figure-model-depth}
\end{figure}

\paragraph{Training Minibatch Size}
\label{batch-size-effect}
Figure \ref{figure-batch-size} shows that BatchNormalNorm
maintains a high level of performance
across
minibatch sizes used during training,
which provides further evidence for normality normalization's general effectiveness across a variety of configurations.

\begin{figure}[h!]
  \centering
  	  \includegraphics[width=0.80\linewidth]{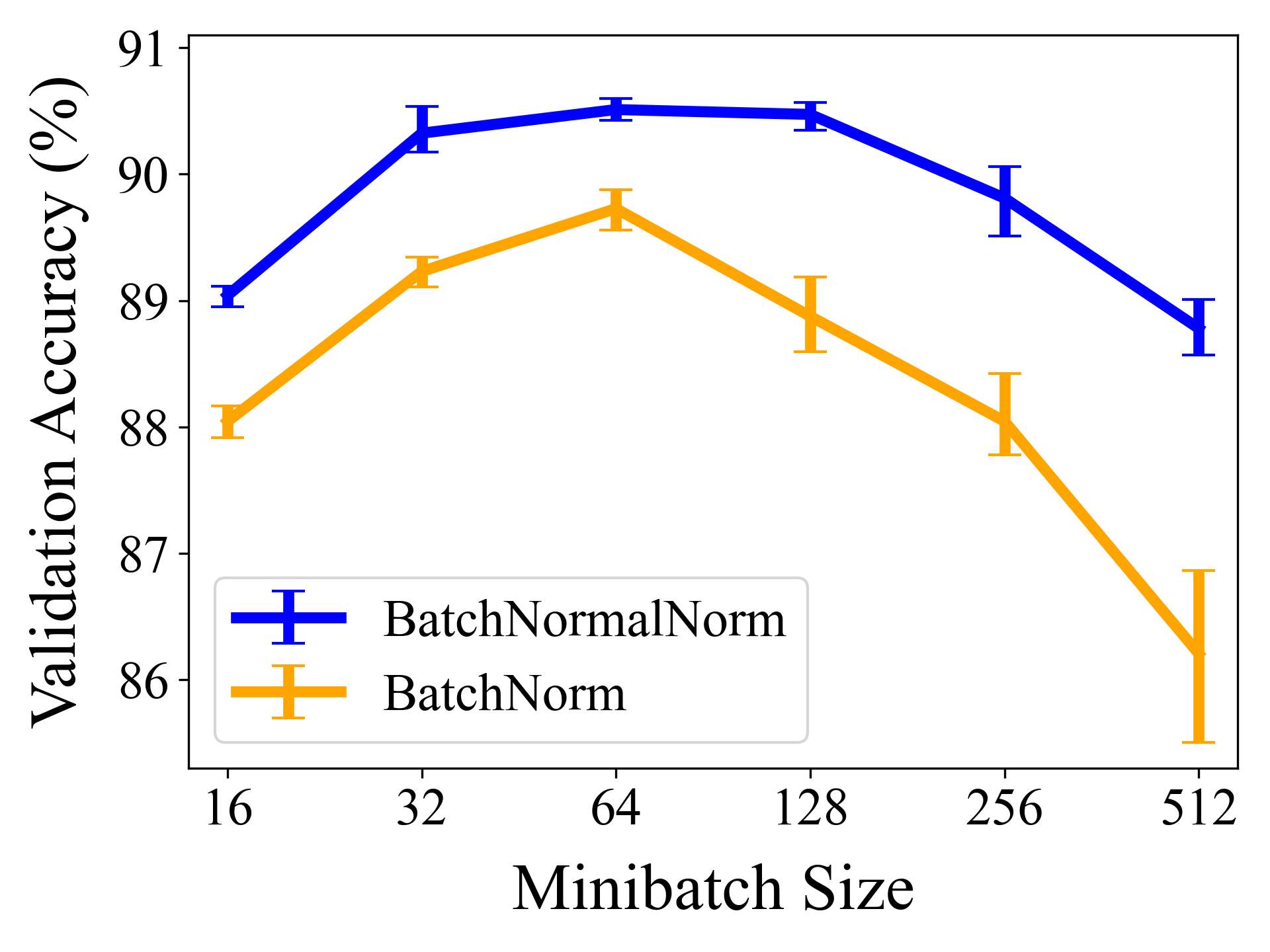}
  	  \vspace{-3.5mm}
    \caption{\textbf{Normality normalization is effective across minibatch sizes used during training.}
Validation accuracy for ResNet18 architectures evaluated on the CIFAR10 dataset,
with varying minibatch sizes used during training,
when using BatchNormalNorm vs. BatchNorm.
    }
    \label{figure-batch-size}
\end{figure}

\begin{figure}[h!]
  \centering
  \begin{minipage}[b]{0.49\textwidth}
  \centering
    \begin{subfigure}[b]{0.32\linewidth}
	  \includegraphics[width=\linewidth]{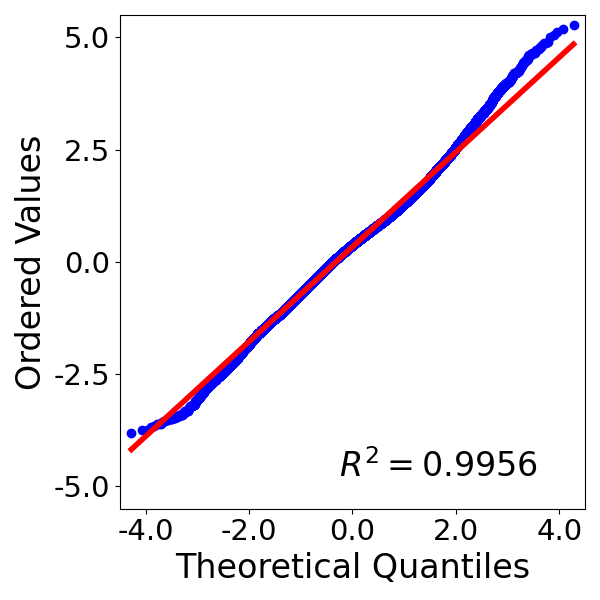}
    \end{subfigure}
    \begin{subfigure}[b]{0.32\linewidth}
      \includegraphics[width=\linewidth]{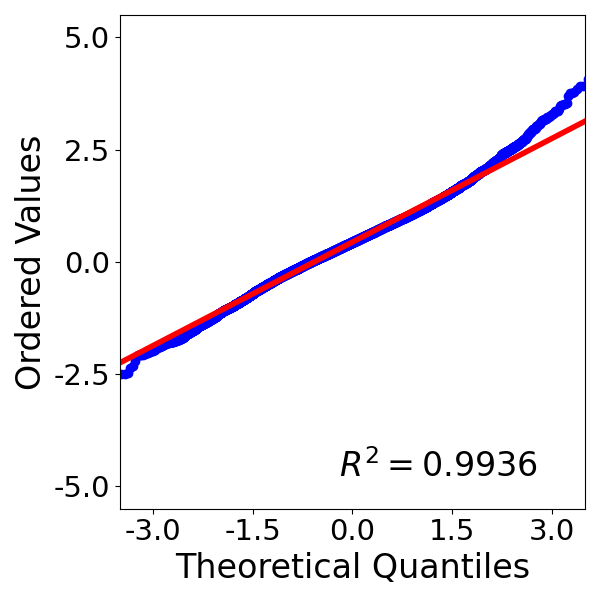}
    \end{subfigure}
    \begin{subfigure}[b]{0.32\linewidth}
      \includegraphics[width=\linewidth]{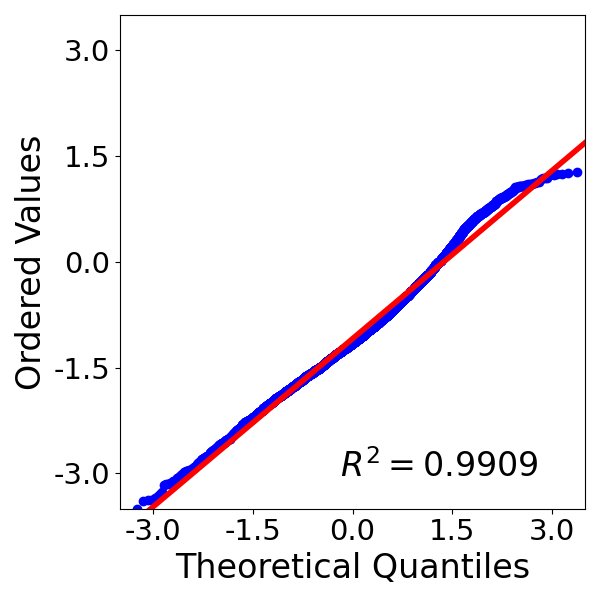}
    \end{subfigure}
  \end{minipage}
  \begin{minipage}[b]{0.49\textwidth}
  \centering
    \begin{subfigure}[b]{0.32\linewidth}
	  \includegraphics[width=\linewidth]{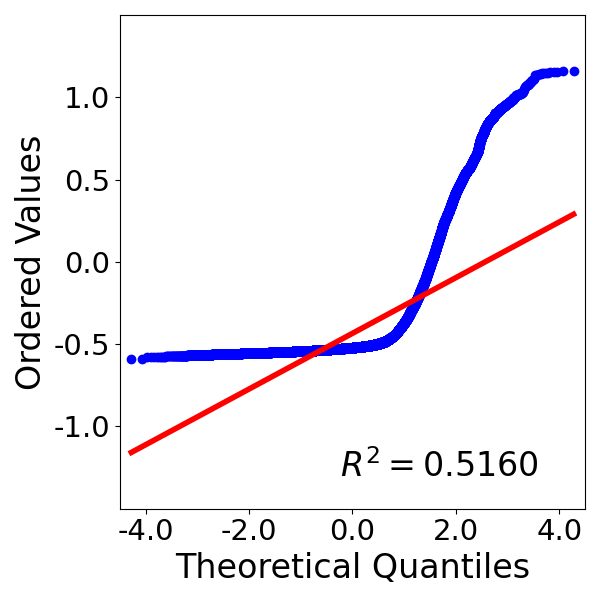}
    \end{subfigure}
    \begin{subfigure}[b]{0.32\linewidth}
      \includegraphics[width=\linewidth]{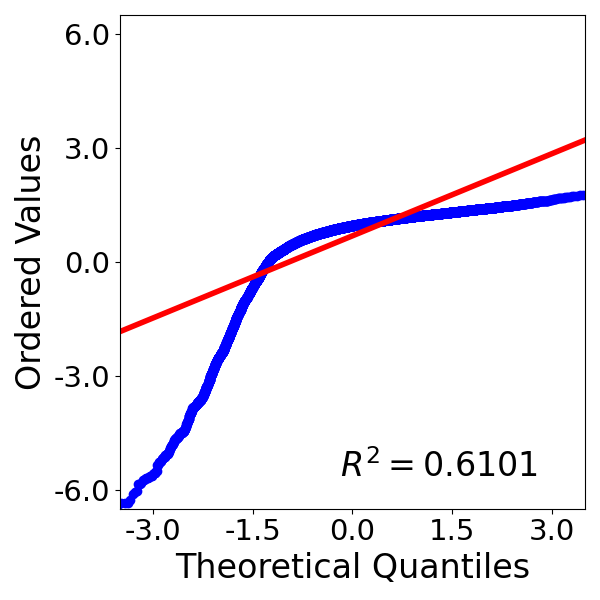}
    \end{subfigure}
    \begin{subfigure}[b]{0.32\linewidth}
      \includegraphics[width=\linewidth]{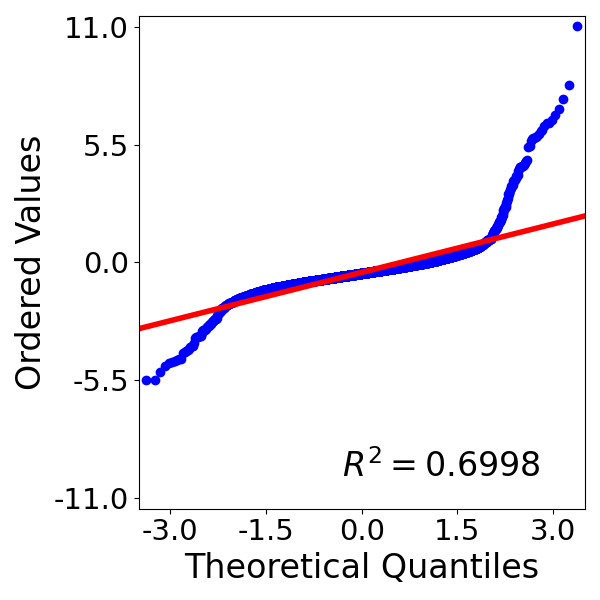}
    \end{subfigure}
  \end{minipage}
    \vspace{-6.5mm}
    \caption{Representative
	Q--Q plots of feature values for
    models
    trained to convergence
    with
    BatchNormalNorm (post-power transform, top row) vs. BatchNorm (post-normalization, bottom row),
    measured for the same validation minibatch (ResNet34/STL10). Left to right: increasing layer number.
    The x-axis represents the theoretical quantiles of the normal distribution,
    and
    the y-axis the
    sample's
    ordered values.
    A higher $R^{2}$ value for the line of best fit signifies greater gaussianity in the features. BatchNormalNorm induces greater gaussianity in the features throughout the model, in comparison to BatchNorm.
    }
    \label{figure-qq}
\end{figure}

\begin{figure}[h!]
  \centering
 \includegraphics[width=0.80\linewidth]{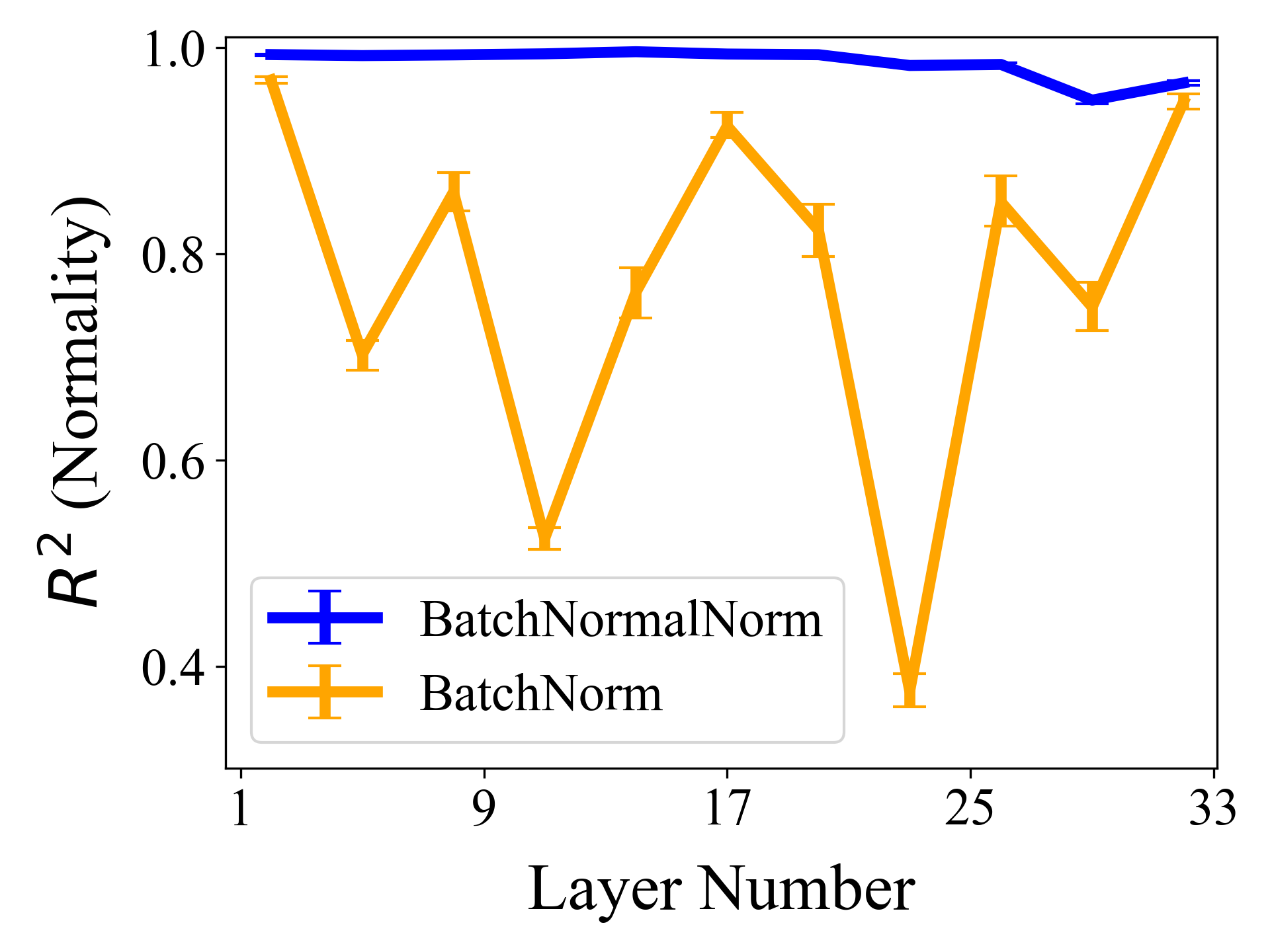}
    \vspace{-3.5mm}
    \caption{The
    average $R^{2}$
    values
    for each model layer,
    derived from
    several
    Q--Q plots
    (see Figure \ref{figure-qq})
    corresponding to
    $20$ channel and
    $10$ validation minibatch combinations,
    for models trained to convergence with BatchNormalNorm vs. BatchNorm.
    The plot demonstrates
    that
    normality normalization
    leads to
    greater gaussianity throughout the model layers.
    }
    \label{figure-r2}
\end{figure}

\pdfbookmark[2]{Normality of Representations}{bookmark-normality-representations}
\subsection{Normality of Representations}
\label{normality-representations}
Figure \ref{figure-qq} shows
representative
Q--Q
plots
\citep{
9ba8158e-d8f3-3f5b-8e59-a801be8dc025},
a method for assessing normality,
for post-power transform feature values when using BatchNormalNorm, and post-normalization
feature
values when using BatchNorm.
Figure \ref{figure-r2} shows
an aggregate measure of normality across model layers,
derived from
several Q--Q plots corresponding to different channel and minibatch combinations.
The figures correspond to models which have been trained to convergence.
The plots demonstrate
that normality normalization leads to greater gaussianity throughout the model layers.

\pdfbookmark[2]{Comparison of Additive Gaussian Noise With Scaling and Gaussian Dropout}{bookmark-experiments-other-noise-techniques}
\subsection{Comparison of Additive Gaussian Noise With Scaling and Gaussian Dropout}
\label{experiments-other-noise-techniques}
Here we contrast the proposed method of additive Gaussian noise with scaling described in Subsection \ref{perturbative-noise},
with
two other noise-based techniques.

The first is Gaussian dropout \citep{Srivastava2014DropoutAS},
where for
each input indexed by $i \in \left\{1, \ldots, N\right\}$, during training we have $y_{i} = x_{i}\cdot\left(1+z_{i} \cdot \sqrt{\frac{1-p}{p}}\right)$, where $x_{i}$ is the $i$-th input's post-power transform value, $z_{i} \sim \mathcal{N}\left(0,1\right)$, and $p \in \left(0,1\right]$ is the retention rate.

The second is additive Gaussian noise,
but without scaling by each channel's minibatch statistics. This corresponds to the proposed method in the case where $s$ is fixed to the mean of a standard half-normal distribution,\footnote{This value
for
$s$
precisely mirrors how
it
is
calculated
in
Algorithm \ref{algorithm1},
since recall
$s = \frac{1}{N}\sum_{i=1}^{N}\abs{x_{i} - \bar{x}}$.
}
i.e. $s = \sqrt{\frac{2}{\pi}}$ across all channels;
and thus does not depend on the channel statistics.

Figure \ref{figure-gaussian-dropout}
shows that
additive Gaussian noise with scaling
is more effective than Gaussian dropout,
giving further evidence for the novelty and utility of the proposed method.
It is also more effective than additive Gaussian noise (without scaling), which suggests the norm of the channel statistics plays an important role when using additive random noise.

\begin{figure}[h!]
  \centering
  	  \includegraphics[width=0.80\linewidth]{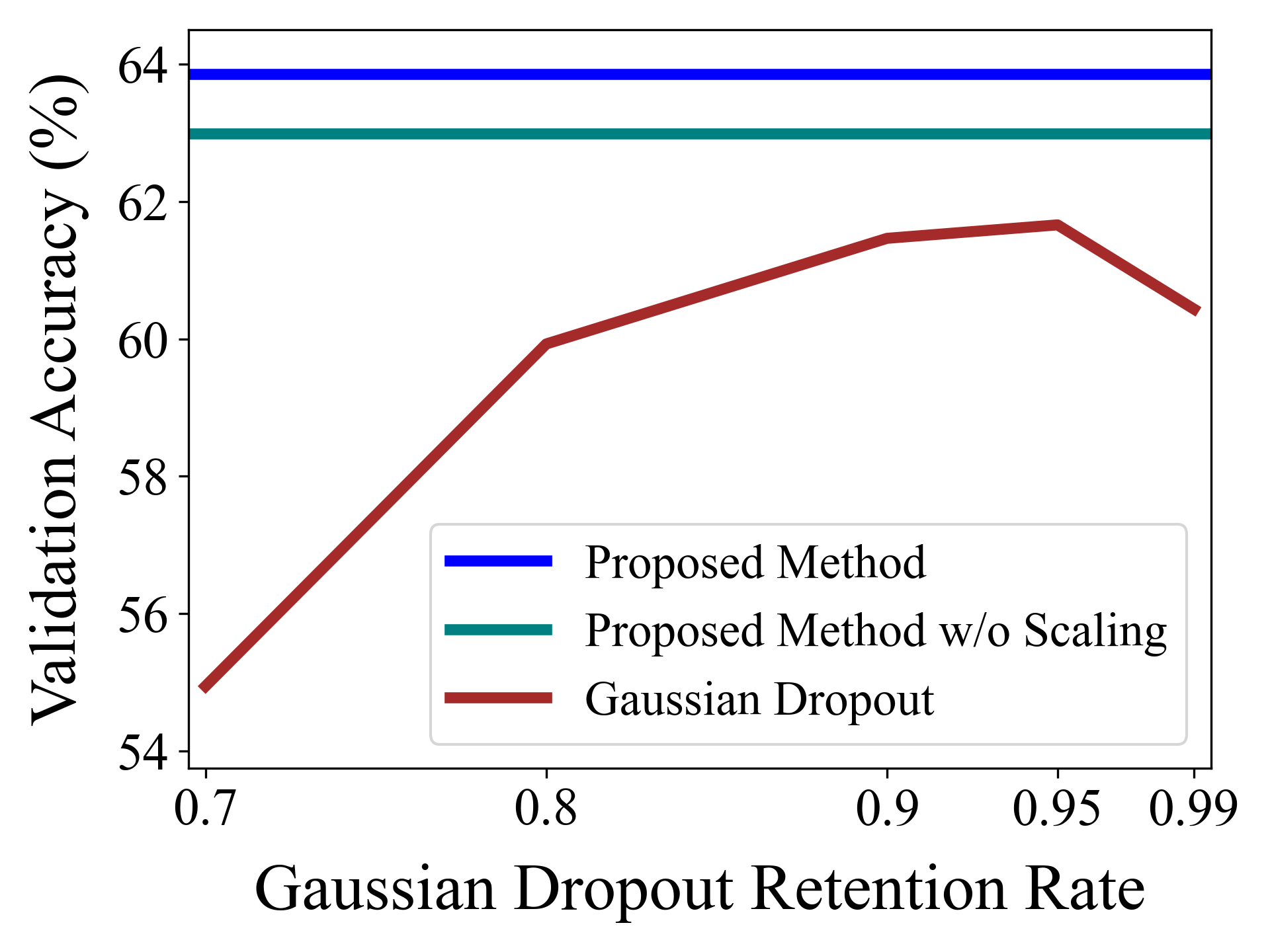}
    \vspace{-3.0mm}
    \caption{\textbf{Additive Gaussian noise with scaling is effective.}
Validation accuracy for models trained with BatchNormalNorm (ResNet34/STL10),
but with varying forms for the noise component of the normalization layer.
    }
    \label{figure-gaussian-dropout}
\end{figure}

One reason
why additive Gaussian noise with scaling
may work
better than Gaussian dropout, is because the latter scales activations multiplicatively, which means the effect of the noise is incorporated in the backpropagated errors. In contrast,
the proposed method's
noise component
does not contribute to the gradient updates directly, because it is additive. This would suggest that models trained with normality normalization obtain higher generalization performance,
because they must become robust to misattribution of gradient values during backpropagation, relative to the corrupted activation values during the forward pass.

\pdfbookmark[2]{Effect of Degree of Gaussianization}{bookmark-experiments-effect-degree-gaussianization}
\subsection{Effect of Degree of Gaussianization}
\label{experiments-effect-degree-gaussianization}
Here we consider
what
effect
differing degrees of gaussianization have on model performance, as measured by the proximity of the estimate $\hat{\lambda}$ to
the
Newton-Raphson
solution,
which was given by Equation \ref{newton-raphson-update}.

We control the proximity to the Newton-Raphson
solution
using a parameter $\alpha \in \left[0,1\right]$ in the following equation
\begin{equation}
\begin{aligned}
	\hat{\lambda} &= 1 -
	\alpha\frac{\mathcal{L}'\!\left(\bm{h};\lambda=1\right)}{\mathcal{L}''\!\left(\bm{h};\lambda=1\right)},
\end{aligned}
\label{newton-raphson-alpha-update}
\end{equation}
where $\alpha=1$ corresponds to the Newton-Raphson
solution,
and decreasing values of $\alpha$ reduce the strength of the gaussianization.

Figure \ref{figure-alpha-effect} demonstrates that the method's performance increases with increasing $\alpha$, and obtains its best performance for $\alpha=1$.
This provides further evidence that increasing gaussianity
improves
model performance.

\begin{figure}[h!]
  \centering
  	  \includegraphics[width=0.80\linewidth]{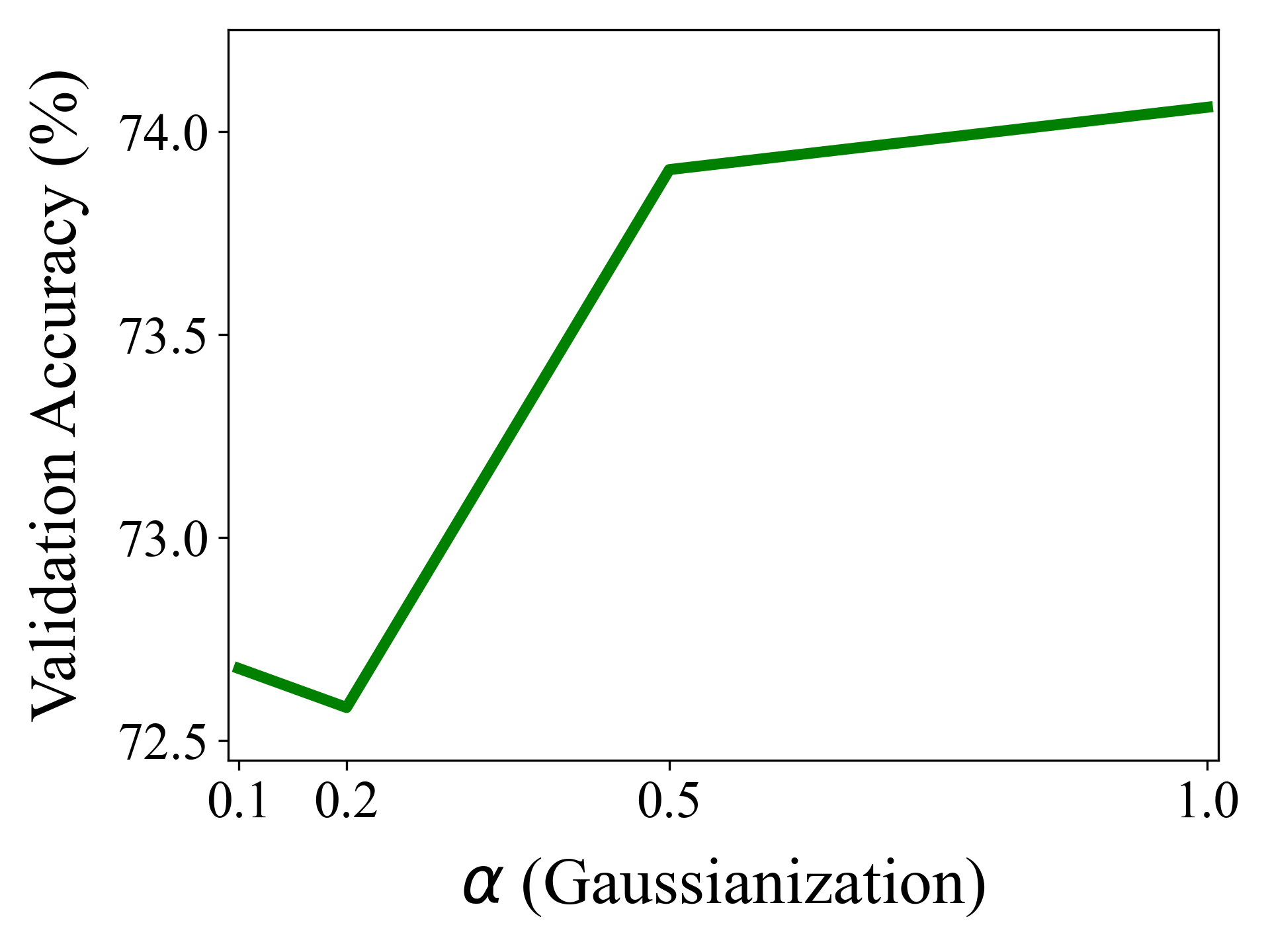}
  	  \vspace{-3.5mm}
    \caption{\textbf{Increasing gaussianity
    improves model performance.}
    Validation accuracy for models trained using
	BatchNormalNorm without noise (ResNet50/Caltech101), and with varying strengths for the gaussianization (parameterized by $\alpha$) when applying the power transform.
    }
    \label{figure-alpha-effect}
\end{figure}

\pdfbookmark[2]{Controlling for the Power Transform and Additive Noise Components}{bookmark-control-effect-norm}
\subsection{Controlling for the Power Transform and Additive Noise Components}
\label{control-effect-norm}
Figure \ref{figure-control-norm-layer}
demonstrates that both components of normality normalization -- the power transform, and the additive Gaussian noise with scaling -- each
contribute meaningfully to the increase in performance
for
models trained with normality normalization.

\begin{figure}[h!]
  \centering
  \begin{minipage}[b]{0.49\textwidth}
  \centering
    \begin{subfigure}[b]{0.49\linewidth}
	  \includegraphics[width=\linewidth]{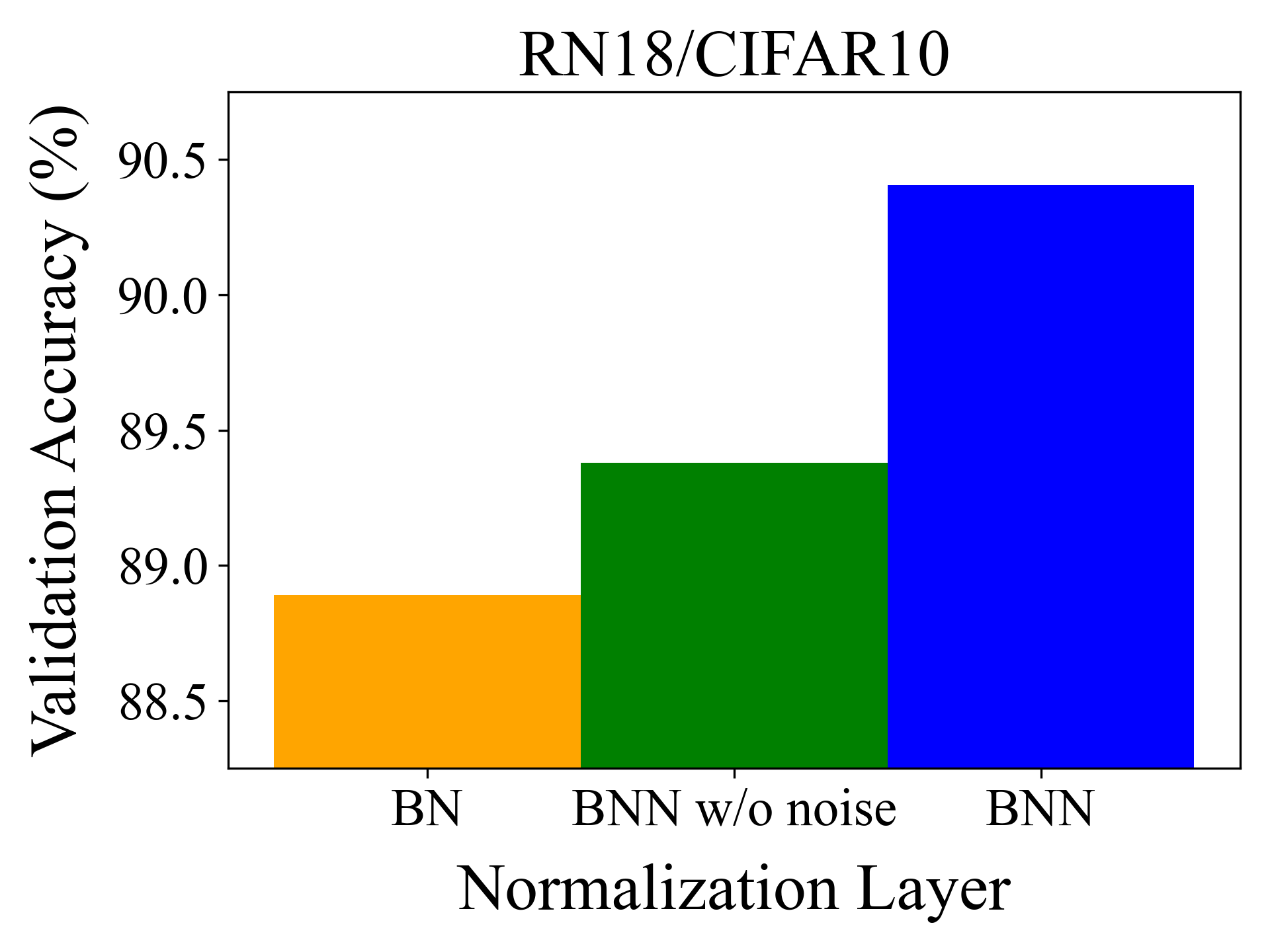}
    \end{subfigure}
    \begin{subfigure}[b]{0.49\linewidth}
      \includegraphics[width=\linewidth]{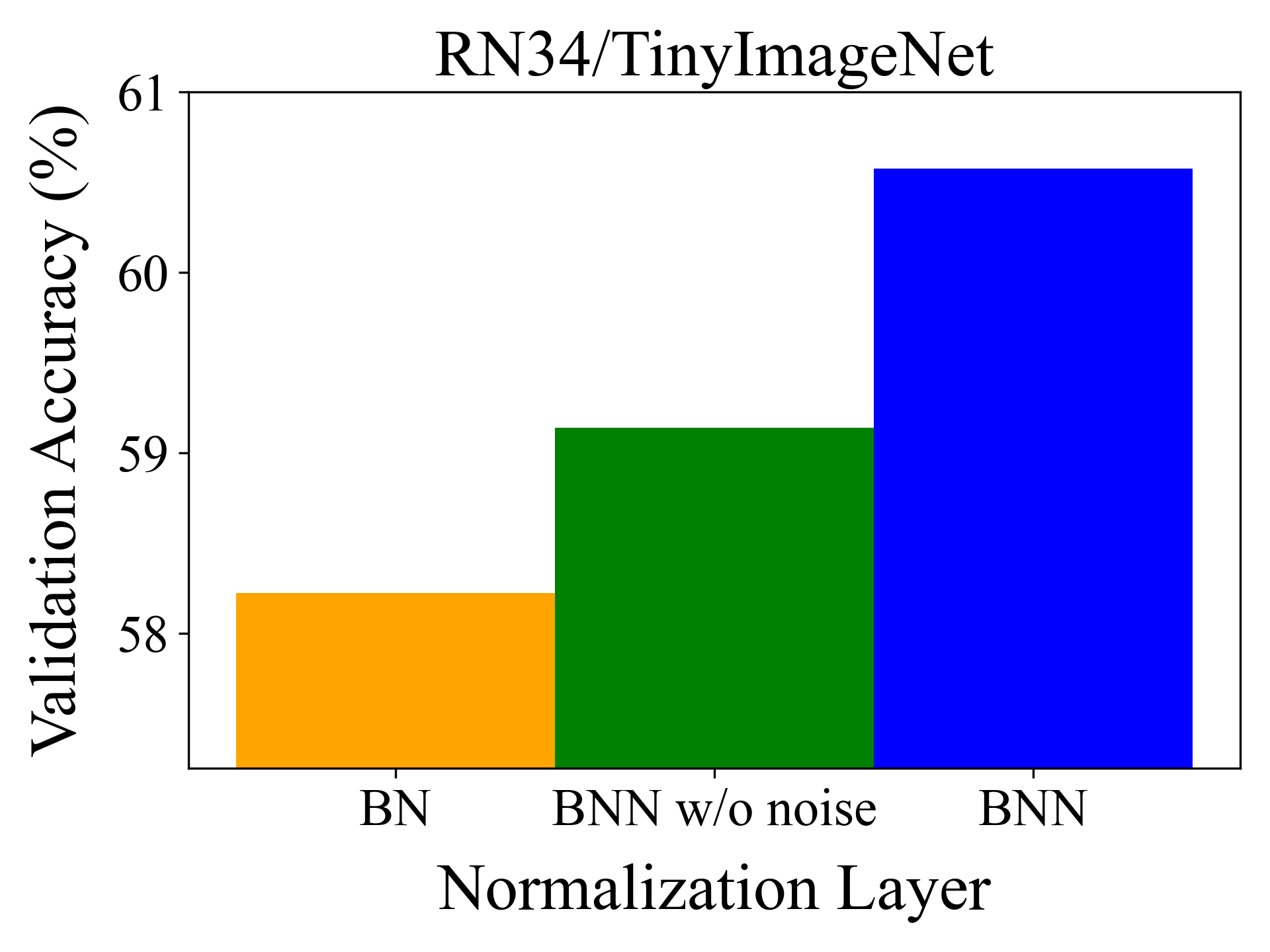}
    \end{subfigure}
  \end{minipage}
  \begin{minipage}[b]{0.49\textwidth}
  \centering
    \begin{subfigure}[b]{0.49\linewidth}
      \includegraphics[width=\linewidth]{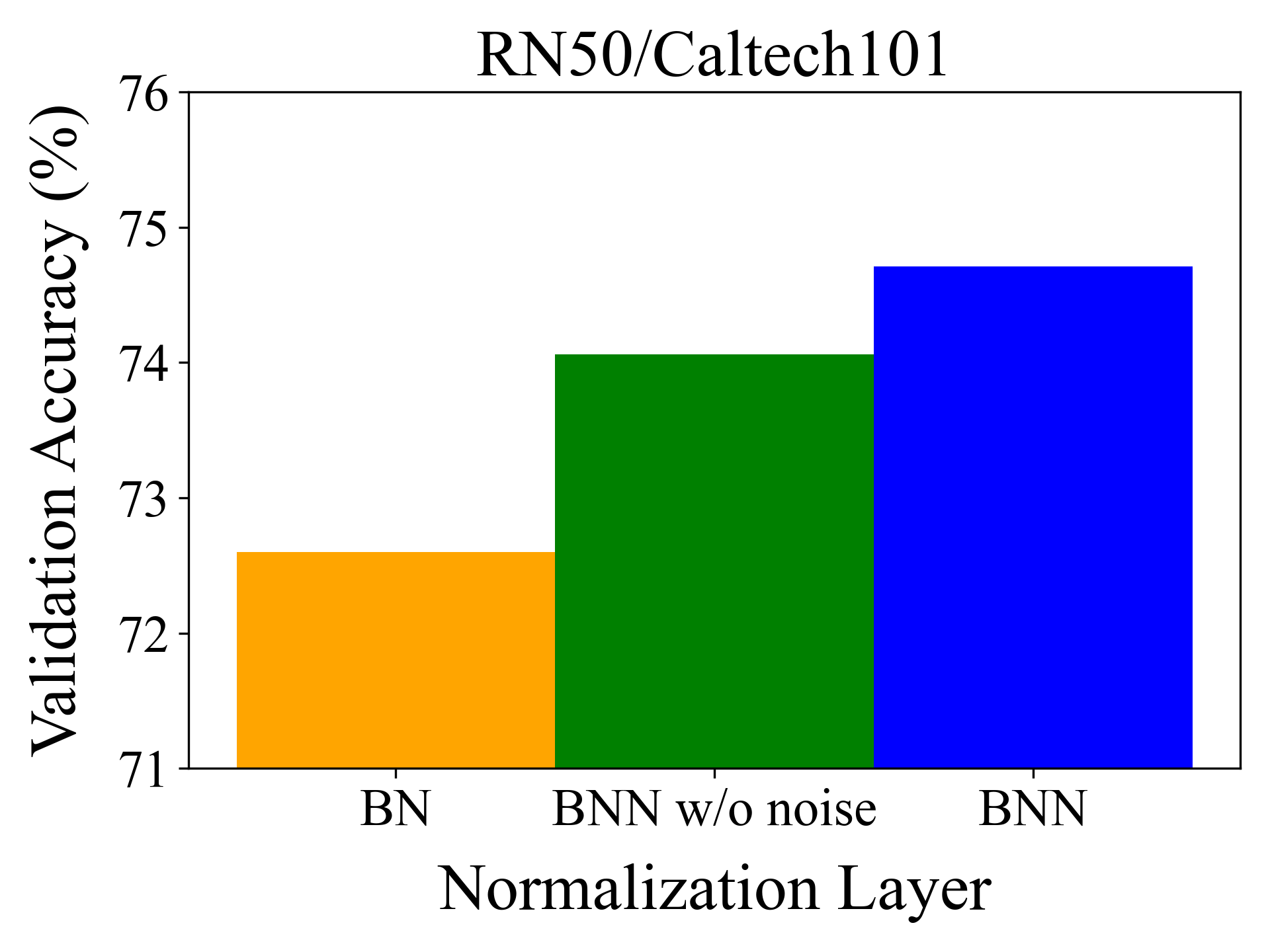}
    \end{subfigure}
    \begin{subfigure}[b]{0.49\linewidth}
      \includegraphics[width=\linewidth]{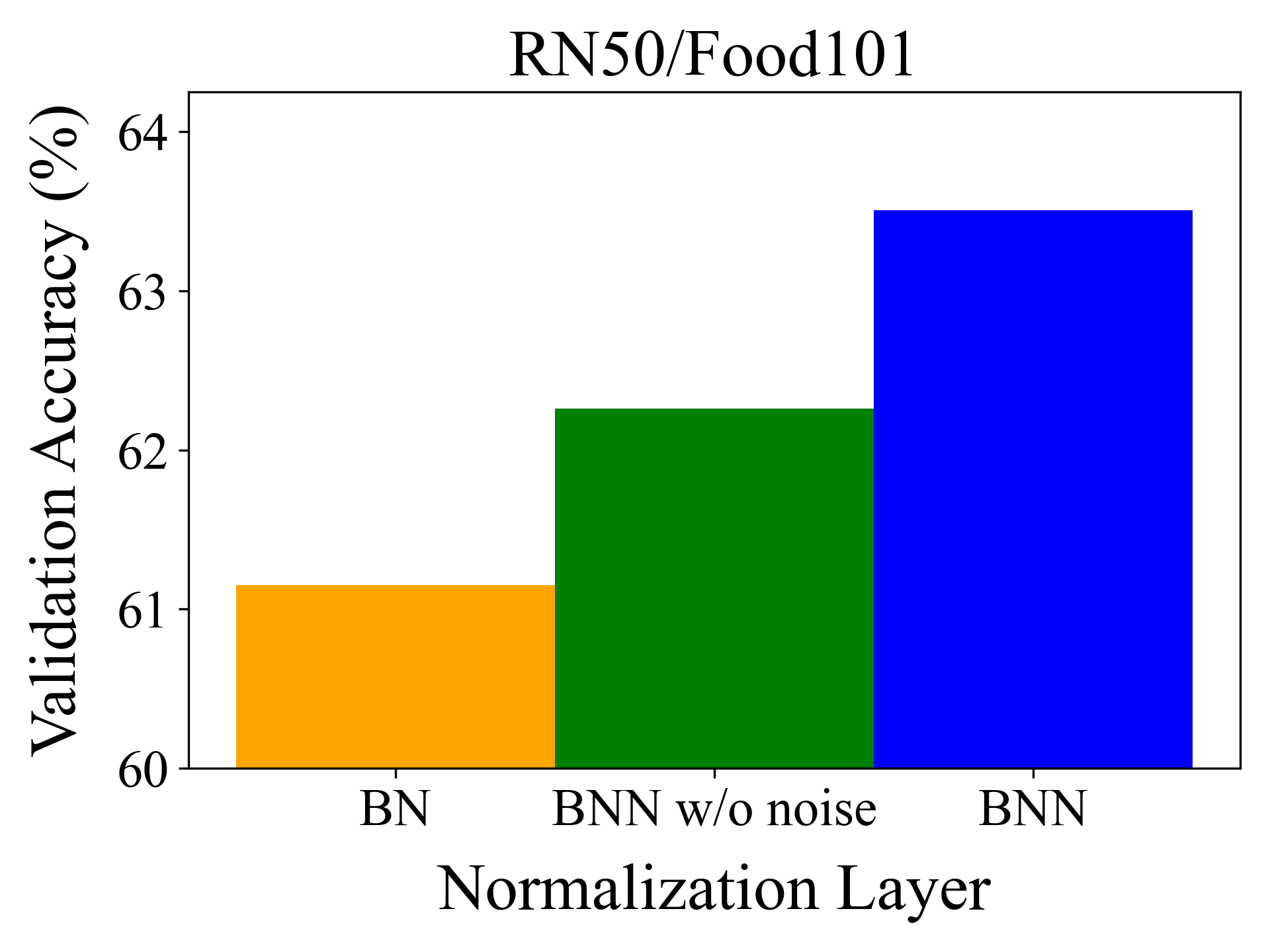}
    \end{subfigure}
  \end{minipage}
    \vspace{-6.5mm}
    \caption{\textbf{Controlling for the effects of the power transform and the additive Gaussian noise with scaling components.}
	Each subplot
	shows
	the performance
	of
	models
	trained with
	BatchNormalNorm with
	the use of additive Gaussian noise with scaling
	(BNN),
	and without (BNN w/o noise), while using BatchNorm (BN) as a baseline. Subplot titles indicate the model and dataset combination.
    }
    \label{figure-control-norm-layer}
\end{figure}

\pdfbookmark[2]{Additional Experiments \& Analysis}{bookmark-text-additional-experiments-analysis}
\subsection{Additional Experiments \& Analysis}
\label{text-additional-experiments-analysis}
We next describe
several additional experiments and analyses which serve to
further
demonstrate the effectiveness of normality normalization,
and to substantiate the applicability of the motivation we presented in Section \ref{motivation}.

\paragraph{Normality normalization induces robustness to noise at test time.}
Appendix \ref{noise-robustness}
demonstrates
that
models trained using normality normalization are more robust to random noise at test time.
This substantiates the applicability of the noise robustness framework
presented
in Motivation Subsection \ref{mutual-information-game}, and consequently of the benefit of gaussianizing
representations.

\paragraph{Speed benchmarks.}
Appendix \ref{appendix-speed-benchmarks}
shows
that normality normalization increases runtime; with a
larger
deviation at training time
than at test time.

\paragraph{Normality normalization uniquely maintains gaussianity throughout training.}
Appendix \ref{appendix-normality-initialization}
provides
a
graphical illustration
of
the fact
that at initialization, layer pre-activations are
close
to
Gaussian regardless of the normalization layer employed;
thus
only models trained with normality normalization maintain gaussianity throughout training.

\paragraph{Normality normalization induces greater feature independence.}
Appendix \ref{appendix-joint-normality-independence}
demonstrates
that normality normalization imbues models with greater joint normality and greater independence between channel features, throughout the layers of a model. This is of value in
context of the
benefit
feature independence is thought to provide, which
was
explored
in Motivation Subsection \ref{text-maximally-independent}.

\pdfbookmark[1]{Related Work \& Future Directions}{bookmark-related-work}
\section{Related Work \& Future Directions}
\label{related-work}
\paragraph{Power Transforms}
Various power transforms have been developed \citep{b6d53586-2890-3ac6-bec5-c3cfdcb64879,Yeo2000}
and their properties studied \citep{10.2307/2287172},
for increasing normality in data. \citet{b6d53586-2890-3ac6-bec5-c3cfdcb64879} defined a power transform which is convex in its parameter, but is only defined for positive variables.
\citet{Yeo2000} presented an alternative
power transform which was
furthermore defined for the entire real line, preserved the convexity property with respect to its parameter for positive input values (concavity in the parameter for negative input values), and additionally addressed skewed input distributions.

It is worth noting that many
power transforms
were developed with the aim of improving the validity of statistical tests relying on the assumption of normality in the data.
This is in contrast with the present work, which uses an information-theoretic motivation for gaussianizing.

\paragraph{Gaussianization}
Alternative
approaches to
gaussianization,
such as
transformations
for gaussianizing heavy-tailed distributions \citep{https://doi.org/10.1155/2015/909231},
iterative gaussianization techniques \citep{NIPS20003c947bc2,5720319},
and copula-based gaussianization \citep{10.5555/1204326},
offer interesting directions for future work.
Non-parametric techniques
for gaussianizing,
for example those using quantile functions
\citep{quantile-book},
may
not be easily amenable to the deep learning setting where models are trained using
backpropagation
and
gradient descent.

\paragraph{Usage in Other Normalization Layers}
Works which have previously assumed normality in the pre-activations
to motivate and develop their methodology,
for example as seen in
normalization propagation \citep{pmlr-v48-arpitb16}, may benefit from normality normalization's explicit gaussianizing effect.
It would also be interesting to explore what effect gaussianizing model weights
might have,
for example by using normality normalization to augment weight normalization \citep{NIPS2016_ed265bc9}.

\paragraph{Adversarial Robustness}
It would be interesting to tie the present work with
those
suggesting robustness to
$\ell_{2}$-norm constrained
adversarial perturbations
increases when training with Gaussian noise \citep{pmlr-v97-cohen19c,NEURIPS20193a24b25a}. Furthermore, it has been suggested
that
adversarial examples and images corrupted with Gaussian noise may be related \citep{Ford2019AdversarialEA}.
This
may
indicate gaining robustness to Gaussian noise not only in the inputs, but throughout the model, can lead to greater adversarial robustness.

However, gaussianizing activations and training with Gaussian noise, may only be a defense in the distributional sense; exact knowledge of the weights (and consequently of the activation values), as is often assumed in the adversarial robustness setting, is not captured by the noise-based robustness framework, which is only concerned with distributional assumptions over the activation values. Nevertheless it does suggest that, on average, greater robustness may be attainable.

\paragraph{Neural Networks as Gaussian Processes}
\citet{10.5555/525544} showed that in the limit of infinite width, a single layer neural network at initialization approximates a Gaussian process. This result
has been extended to the multi-layer setting by \citet{lee2018deep}, and
\citet{NEURIPS20185a4be1fa,NEURIPS20190d1a9651} suggest the Gaussian process approximation may remain valid beyond network initialization.
However, these analyses still necessitate
the infinite width limit assumption.

Subsequent work
showed that
batch normalization
lends
itself to a
non-asymptotic
approximation
to
normality
throughout the layers of neural networks
at initialization \citep{NEURIPS202126cd8eca}.
Given its gaussianizing effect,
layers trained with
normality normalization may be amenable -- throughout training -- to a non-asymptotic approximation to Gaussian processes.
This
could
help
further
address
the
disparity
in
the analysis of neural networks in the infinite width limit, for example as in mean-field theory,
with
the finite width setting
\citep{pmlr-v202-joudaki23a}.

\pdfbookmark[1]{Conclusion}{bookmark-conclusion}
\section{Conclusion}
\label{conclusion}
Among
the methodological developments that have spurred the advent of deep learning,
their success
has often been attributed to their effect
on
the model's ability to learn and encode representations effectively, whether in the activations or in the weights.
This can be seen, for example, by considering the importance
attributed to
initializing model weights suitably, or by the effect different activation functions have on learning dynamics.

Seldom
has a prescription for precisely what distribution a deep learning model should use to effectively encode its activations,
and exactly how this can be achieved,
been investigated. The present work addresses this -- first by motivating the normal distribution as the probability distribution of choice, and subsequently by materializing this choice through normality normalization.

It is perhaps nowhere clearer what representational benefit normality normalization provides, than when considering that no additional learnable parameters
relative to existing normalization layers
were introduced. This highlights -- and precisely controls for the effect of -- the importance of encouraging models to encode their representations effectively.

We presented normality normalization: a novel, principledly motivated, normalization layer.
Our experiments and analysis comprehensively demonstrated the effectiveness of normality normalization,
in regards to its generalization performance on an array of widely used model and dataset combinations,
its consistently strong performance across
various common factors of variation
such as model
width, depth, and training minibatch size,
its suitability
for usage
wherever existing normalization layers are conventionally used,
and through its effect on improving model robustness to random perturbations.
\newpage

\section*{Acknowledgments}
We acknowledge the support
provided by
Compute Ontario (computeontario.ca) and the Digital Research Alliance of Canada (alliancecan.ca),
and
the support of
the Natural Sciences and Engineering Research Council of Canada (NSERC),
Discovery Grant RGPIN-2021-02527.

\section*{Impact Statement}
This work is of general interest to the machine learning
and broader scientific
community.
There are many potential applications of the work,
for which
endeavoring to
judiciously
highlight
one such possible application, in place of another,
would be
a
precarious
undertaking.

\bibliography{icml2025}
\bibliographystyle{icml2025}

\newpage
\appendix
\onecolumn
\pdfbookmark[1]{Additional Experiments \& Analysis}{bookmark-appendix-additional-results}
\section{Additional Experiments \& Analysis}

\pdfbookmark[2]{Noise Robustness}{bookmark-noise-robustness}
\subsection{Noise Robustness}
\label{noise-robustness}
We use the following framework to measure a model's robustness to noise (a similar setting is used by \citet{Arora2018StrongerGB}). For a given data point, consider a
pair of units in a neural network, the first in the $k$-th layer and the second in the $\ell$-th layer. For the
unit in the $k$-th layer, let $x$ denote the data point's post-normalization value. Let $\phi_{k,\ell}\left(x\right)$ be the same data point's post-normalization value for the
unit in the subsequent layer $\ell$, where the function $\phi_{k,\ell}$ encapsulates all the intermediate computations between the two normalization layers $k$ and $\ell$.

Let
$y = x + z \cdot \delta \cdot \frac{1}{N}\norm{\bm{x} - \bar{\bm{x}}}_{1}$,
where
as in Subsection \ref{perturbative-noise}
\(z \sim \mathcal{N}\left(0, 1\right)\),
\(\delta \ge 0\),
and here \(\norm{\bm{x} - \bar{\bm{x}}}_{1}\) represents a global estimate for the zero-centered norm of the post-normalized values,
derived from the training set in its entirety.
We then define
noise robustness as follows:

\begin{table*}[h!]
\centering
\caption{\textbf{Normality normalization
induces
	robustness
	to noise at test time.}
    Evaluation of robustness to noise, using the relative error $\zeta_{k,\ell}^{\delta}$ for various layers $k$ and $\ell$, for
    various
    models trained with BatchNormalNorm and BatchNorm.
    Models
    were evaluated using noise factor $\delta=0.5$.
    Top: ResNet18/CIFAR10, middle: ResNet18/CIFAR100, bottom: ResNet34/STL10.
	The layer \(k\)
	at
	which noise is added is denoted on the left side of each row, and each column
	denotes a subsequent layer
	\(\ell\).
    For each entry,
    \(\zeta_{k,\ell}^{\delta}\) was averaged over the entire validation set, and over all channels in the \(k\)-th and \(\ell\)-th layers. This was
    subsequently
    averaged
    across
    \(T=6\) Monte Carlo draws for the random noise,
    and the values presented are furthermore the
    average
    across each of the \(M=6\) trained models.
	In each table entry, the top value represents the relative error for BatchNormalNorm (BNN),
	and the bottom value for BatchNorm (BN),
	with the best value
	shown in bold; lower is better.
    The tables provide
    evidence that models trained
    using
    normality normalization
    are
    generally
    more robust to random noise at test time.
}
\noindent
\vspace{+0.1mm}
\begin{minipage}[t]{1.0\textwidth}
\centering
\begin{tabular}{cc}
\begin{tabular}{ccc|c|c|c|}
&
&
\multicolumn{1}{c}{L5} & \multicolumn{1}{c}{L9} & \multicolumn{1}{c}{L13} & \multicolumn{1}{c}{L17} \\
\cline{3-6}
\multicolumn{1}{c}{L1}
&\multicolumn{1}{c|}{\begin{tabular}{@{}c@{}}BNN\\BN\end{tabular}}
&\begin{tabular}{@{}c@{}}\textbf{0.044 \(\pm\) 0.002} \\ 0.160 $\pm$ 0.009\end{tabular}
&\begin{tabular}{@{}c@{}}\textbf{0.076 \(\pm\) 0.002} \\ 0.266 \(\pm\) 0.014\end{tabular}
&\begin{tabular}{@{}c@{}}\textbf{0.124 \(\pm\) 0.002} \\ 0.390 $\pm$ 0.016\end{tabular}
&\begin{tabular}{@{}c@{}}\textbf{0.211 \(\pm\) 0.005} \\ 1.002 \(\pm\) 0.032\end{tabular}
\\
\cline{3-6}
\multicolumn{1}{c}{L5}
&\multicolumn{1}{c}{}
&\multicolumn{1}{c|}{} &\begin{tabular}{@{}c@{}}\textbf{0.031 \(\pm\) 0.001} \\ 0.060 \(\pm\) 0.002\end{tabular}
&\begin{tabular}{@{}c@{}}\textbf{0.049 \(\pm\) 0.001} \\ 0.093 \(\pm\) 0.004\end{tabular}
&\begin{tabular}{@{}c@{}}\textbf{0.080 \(\pm\) 0.002} \\ 0.201 \(\pm\) 0.005\end{tabular}
\\
\cline{4-6}
\multicolumn{1}{c}{L9}
&\multicolumn{1}{c}{}
&\multicolumn{1}{c}{}
&\multicolumn{1}{c|}{}
&\begin{tabular}{@{}c@{}}\textbf{0.048 \(\pm\) 0.001} \\ 0.060 \(\pm\) 0.003\end{tabular}
&\begin{tabular}{@{}c@{}}\textbf{0.070 \(\pm\) 0.001} \\ 0.117 \(\pm\) 0.003\end{tabular}
\\
\cline{5-6}
\multicolumn{1}{c}{L13}
&\multicolumn{1}{c}{}
&\multicolumn{1}{c}{}
&\multicolumn{1}{c}{}
&\multicolumn{1}{c|}{}
&\begin{tabular}{@{}c@{}}\textbf{0.142 \(\pm\) 0.006} \\ 0.203 \(\pm\) 0.005\end{tabular}
\\
\cline{6-6}
\end{tabular}
\end{tabular}
\end{minipage}
\noindent
\begin{minipage}[t]{1.0\textwidth}
\centering
\vspace{+4mm}
\begin{tabular}{cc}
\begin{tabular}{ccc|c|c|c|}
&
&
\multicolumn{1}{c}{L5} & \multicolumn{1}{c}{L9} & \multicolumn{1}{c}{L13} & \multicolumn{1}{c}{L17} \\
\cline{3-6}
\multicolumn{1}{c}{L1}
&\multicolumn{1}{c|}{\begin{tabular}{@{}c@{}}BNN\\BN\end{tabular}}
&\begin{tabular}{@{}c@{}}\textbf{0.051 \(\pm\) 0.001} \\ 0.177 $\pm$ 0.007\end{tabular}
&\begin{tabular}{@{}c@{}}\textbf{0.076 \(\pm\) 0.001} \\ 0.333 \(\pm\) 0.011\end{tabular}
&\begin{tabular}{@{}c@{}}\textbf{0.100 \(\pm\) 0.001} \\ 0.419 $\pm$ 0.021\end{tabular}
&\begin{tabular}{@{}c@{}}\textbf{0.387 \(\pm\) 0.004} \\ 1.922 \(\pm\) 0.076\end{tabular}
\\
\cline{3-6}
\multicolumn{1}{c}{L5}
&\multicolumn{1}{c}{}
&\multicolumn{1}{c|}{} &\begin{tabular}{@{}c@{}}\textbf{0.027 \(\pm\) 0.002} \\ 0.059 \(\pm\) 0.009\end{tabular}
&\begin{tabular}{@{}c@{}}\textbf{0.038 \(\pm\) 0.003} \\ 0.073 \(\pm\) 0.009\end{tabular}
&\begin{tabular}{@{}c@{}}\textbf{0.149 \(\pm\) 0.012} \\ 0.373 \(\pm\) 0.041\end{tabular}
\\
\cline{4-6}
\multicolumn{1}{c}{L9}
&\multicolumn{1}{c}{}
&\multicolumn{1}{c}{}
&\multicolumn{1}{c|}{}
&\begin{tabular}{@{}c@{}}\textbf{0.044 \(\pm\) 0.001} \\ 0.063 \(\pm\) 0.003\end{tabular}
&\begin{tabular}{@{}c@{}}\textbf{0.151 \(\pm\) 0.002} \\ 0.249 \(\pm\) 0.009\end{tabular}
\\
\cline{5-6}
\multicolumn{1}{c}{L13}
&\multicolumn{1}{c}{}
&\multicolumn{1}{c}{}
&\multicolumn{1}{c}{}
&\multicolumn{1}{c|}{}
&\begin{tabular}{@{}c@{}}\textbf{0.257 \(\pm\) 0.001} \\ 0.367 \(\pm\) 0.020\end{tabular}
\\
\cline{6-6}
\end{tabular}
\end{tabular}
\end{minipage}
\noindent
\begin{minipage}[t]{1.0\textwidth}
\centering
\vspace{+4mm}
\begin{tabular}{cc}
\begin{tabular}{ccc|c|c|c|}
&
&
\multicolumn{1}{c}{L9} & \multicolumn{1}{c}{L17} & \multicolumn{1}{c}{L25} & \multicolumn{1}{c}{L33} \\
\cline{3-6}
\multicolumn{1}{c}{L1}
&\multicolumn{1}{c|}{\begin{tabular}{@{}c@{}}BNN\\BN\end{tabular}}
&\begin{tabular}{@{}c@{}}\textbf{0.373 \(\pm\) 0.018} \\ 0.615 \(\pm\) 0.046\end{tabular}
&\begin{tabular}{@{}c@{}}\textbf{0.709 \(\pm\) 0.034} \\ 1.280 \(\pm\) 0.080\end{tabular}
&\begin{tabular}{@{}c@{}}\textbf{0.452 \(\pm\) 0.044} \\ 1.900 \(\pm\) 0.537\end{tabular}
&\begin{tabular}{@{}c@{}}\textbf{0.565 \(\pm\) 0.053} \\ 2.621 \(\pm\) 0.257\end{tabular}
\\
\cline{3-6}
\multicolumn{1}{c}{L9}
&\multicolumn{1}{c}{}
&\multicolumn{1}{c|}{} &\begin{tabular}{@{}c@{}}0.141 \(\pm\) 0.004 \\ \textbf{0.120 \(\pm\) 0.005}\end{tabular}
&\begin{tabular}{@{}c@{}}\textbf{0.080 \(\pm\) 0.003} \\ 0.099 \(\pm\) 0.011\end{tabular}
&\begin{tabular}{@{}c@{}}\textbf{0.099 \(\pm\) 0.007} \\ 0.307 \(\pm\) 0.013\end{tabular}
\\
\cline{4-6}
\multicolumn{1}{c}{L17}
&\multicolumn{1}{c}{}
&\multicolumn{1}{c}{}
&\multicolumn{1}{c|}{}
&\begin{tabular}{@{}c@{}}\textbf{0.102 \(\pm\) 0.006} \\ 0.104 \(\pm\) 0.006\end{tabular}
&\begin{tabular}{@{}c@{}}\textbf{0.121 \(\pm\) 0.011} \\ 0.324 \(\pm\) 0.012\end{tabular}
\\
\cline{5-6}
\multicolumn{1}{c}{L25}
&\multicolumn{1}{c}{}
&\multicolumn{1}{c}{}
&\multicolumn{1}{c}{}
&\multicolumn{1}{c|}{}
&\begin{tabular}{@{}c@{}}\textbf{0.051 \(\pm\) 0.006} \\ 0.120 \(\pm\) 0.006\end{tabular}
\\
\cline{6-6}
\end{tabular}
\end{tabular}
\end{minipage}
\label{table-relative-error}
\end{table*}

\begin{definition}[Noise Robustness]
\label{definition-noise-robustness}
For
given realization of the
noise sample
\(z\),
let \(\zeta_{k,\ell}^{\delta}\left(x, y\right)\)-robustness be defined as
\begin{equation}
	\begin{aligned}
		\zeta_{k,\ell}^{\delta}\left(x, y\right)
		&\coloneq \frac{\norm{\phi_{k,\ell}\left(x\right) - \phi_{k,\ell}\left(y\right)}_{1}}{\norm{\phi_{k,\ell}\left(x\right)}_{1}}.
	\end{aligned}
\label{noise-robustness-definition}
\end{equation}
\end{definition}

Thus \(\zeta_{k,\ell}^{\delta}\left(x, y\right)\) measures the relative discrepancy between \(\phi_{k,\ell}\left(x\right)\) and \(\phi_{k,\ell}\left(y\right)\) when noise factor \(\delta\) is used, and
effectively represents the noise's attenuation from layer \(k\) to layer \(\ell\). Averaging \(\zeta_{k,\ell}^{\delta}\left(x, y\right)\) over all data points, and over all units in the \(k\)-th and \(\ell\)-th layers, leads to a consolidated estimate of the noise robustness.

\begin{wrapfigure}[20]{r}{0.48\textwidth}
  \centering
  \begin{minipage}[t]{0.47\textwidth}
    \begin{subfigure}[b]{0.48\linewidth}
	  \includegraphics[width=\linewidth]{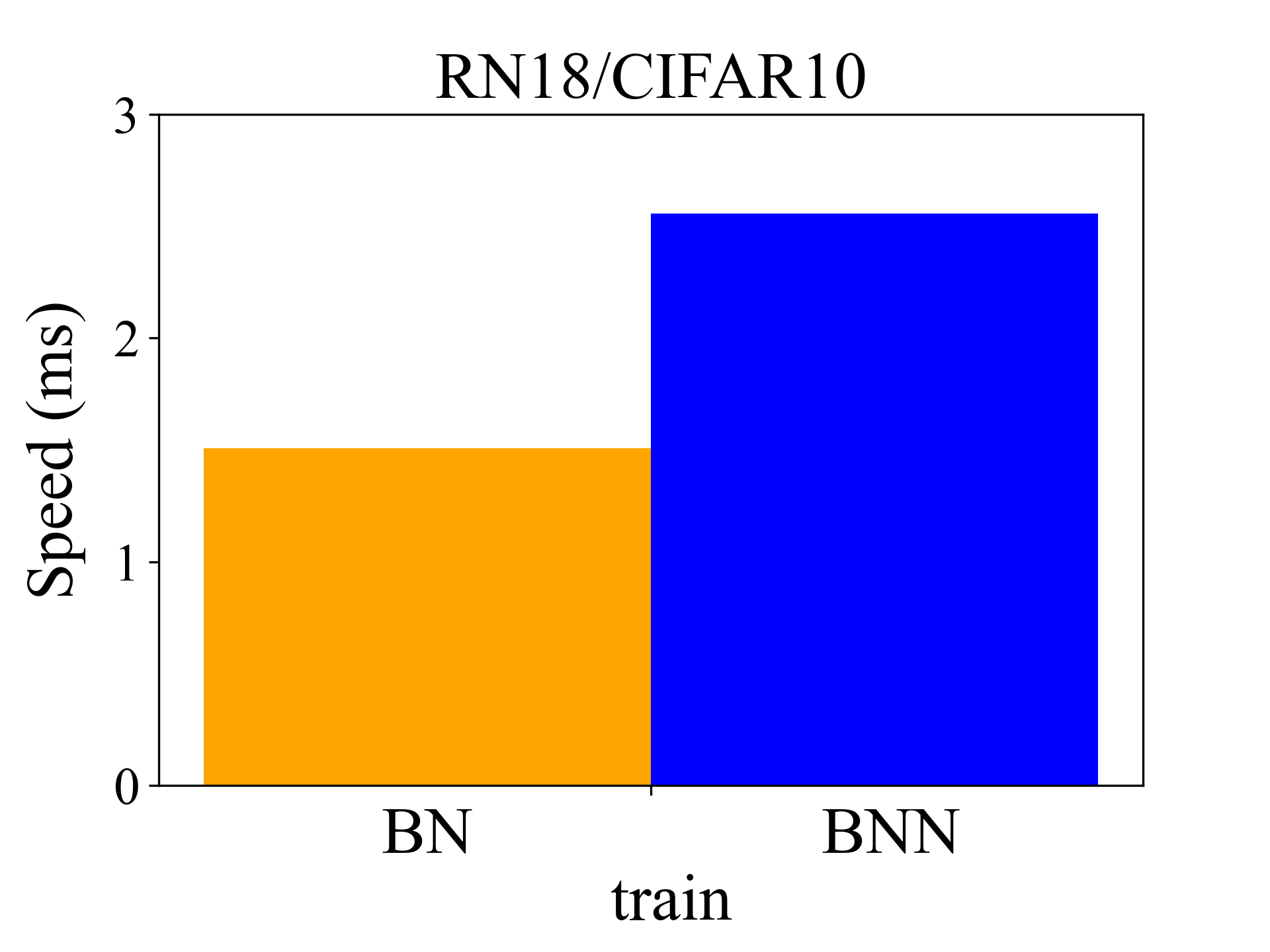}
    \end{subfigure}
    \begin{subfigure}[b]{0.48\linewidth}
      \includegraphics[width=\linewidth]{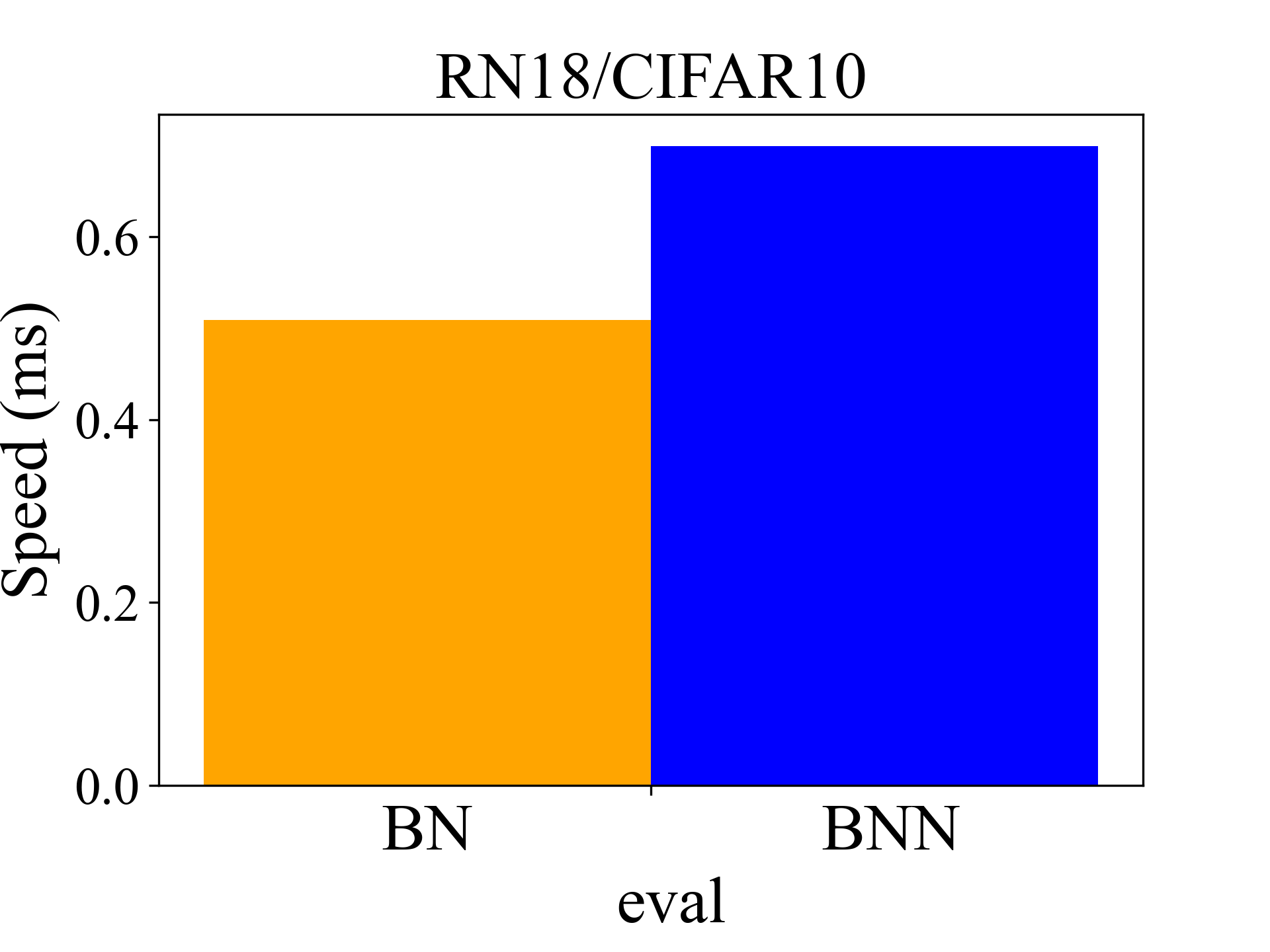}
    \end{subfigure}
  \end{minipage}
  \begin{minipage}[t]{0.47\textwidth}
    \begin{subfigure}[b]{0.48\linewidth}
      \includegraphics[width=\linewidth]{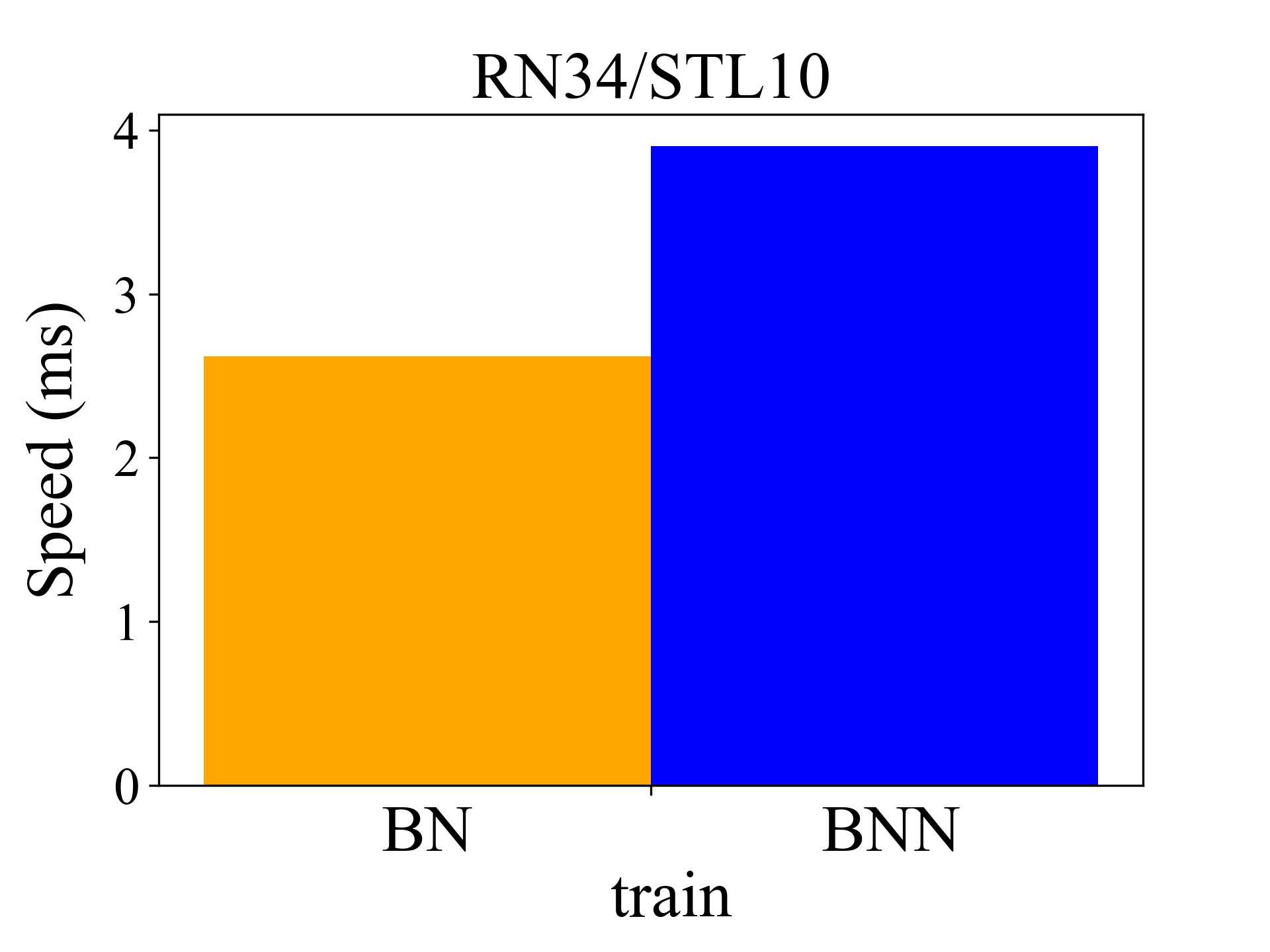}
    \end{subfigure}
    \begin{subfigure}[b]{0.48\linewidth}
      \includegraphics[width=\linewidth]{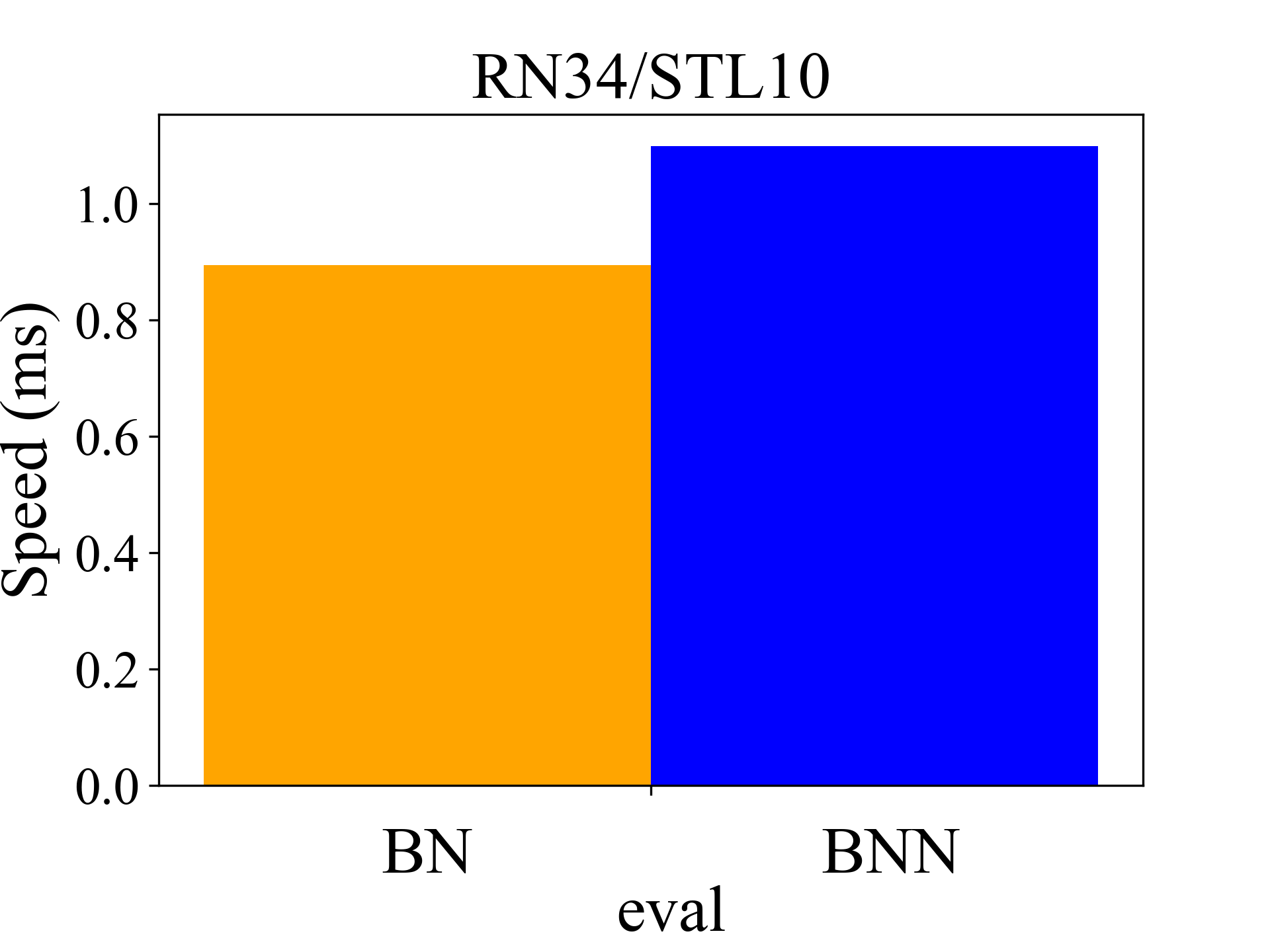}
    \end{subfigure}
    \vspace{-3.4mm}
    \caption{Runtime comparison between models using BatchNormalNorm (BNN) and BatchNorm (BN) for two sets of model \& dataset combinations; top: ResNet18/CIFAR10, bottom: ResNet34/STL10.
    The left hand plot shows the running time during training, and the right hand plot shows the running time during evaluation. See text for details.
    }
    \label{figure-speed}
  \end{minipage}
\end{wrapfigure}
Table \ref{table-relative-error}
demonstrates
the increased robustness to noise obtained when using
BatchNormalNorm
in comparison to
BatchNorm.
This
substantiates the applicability of the noise robustness
framework presented in
Motivation
Subsection \ref{mutual-information-game},
and
consequently of
the benefit of gaussianizing
representations
in normality normalization.

\pdfbookmark[2]{Speed Benchmarks}{bookmark-appendix-speed-benchmarks}
\subsection{Speed Benchmarks}
\label{appendix-speed-benchmarks}
Figure \ref{figure-speed} shows the average per-sample running time for models using BatchNormalNorm and BatchNorm. The values are calculated by taking the average minibatch runtime at train/evaluation time, for the entire training/validation set, then normalizing by the number of samples in the minibatch. Values are obtained using an NVIDIA V100 GPU.
For the purposes of these
benchmarks,
synchronization between the CPU and GPU was enforced.

The plots show
that normality normalization increases runtime;
with a larger deviation at training time than at test time.
It is worth noting however, that
the present work serves as a foundation, both conceptual and methodological, for future works which may continue to leverage the benefits of gaussianizing. We believe improvements to the runtime of normality normalization can be obtained
by leveraging approximations to the operations performed in the present form of normality normalization; for example the operations $\log\left(1+h\right)$,
and raising to the power.

\begin{figure*}[h!]
  \centering
  \begin{minipage}[b]{0.52\textwidth}
  \centering
    \begin{subfigure}[b]{0.32\linewidth}
	  \includegraphics[width=\linewidth]{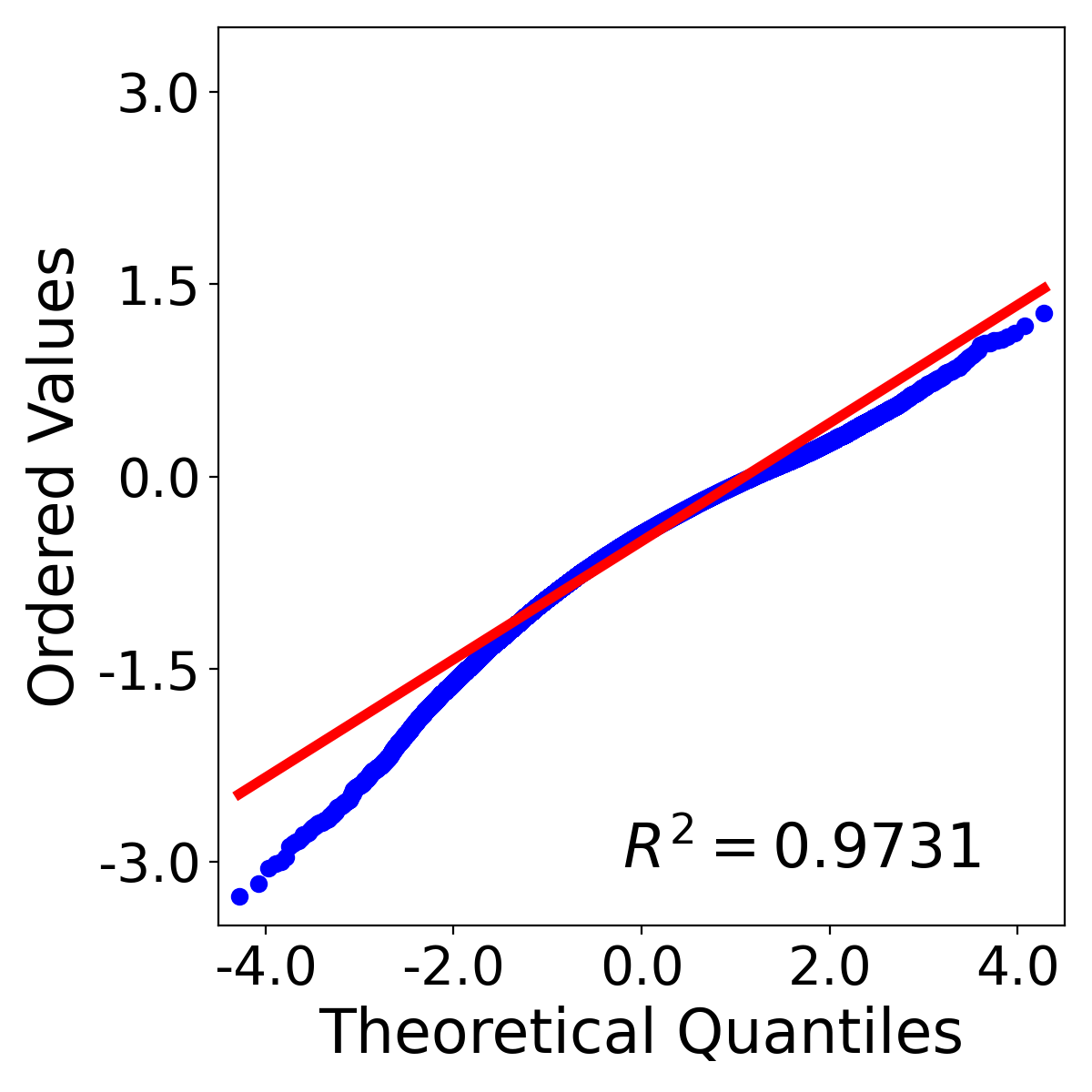}
    \end{subfigure}
    \begin{subfigure}[b]{0.32\linewidth}
      \includegraphics[width=\linewidth]{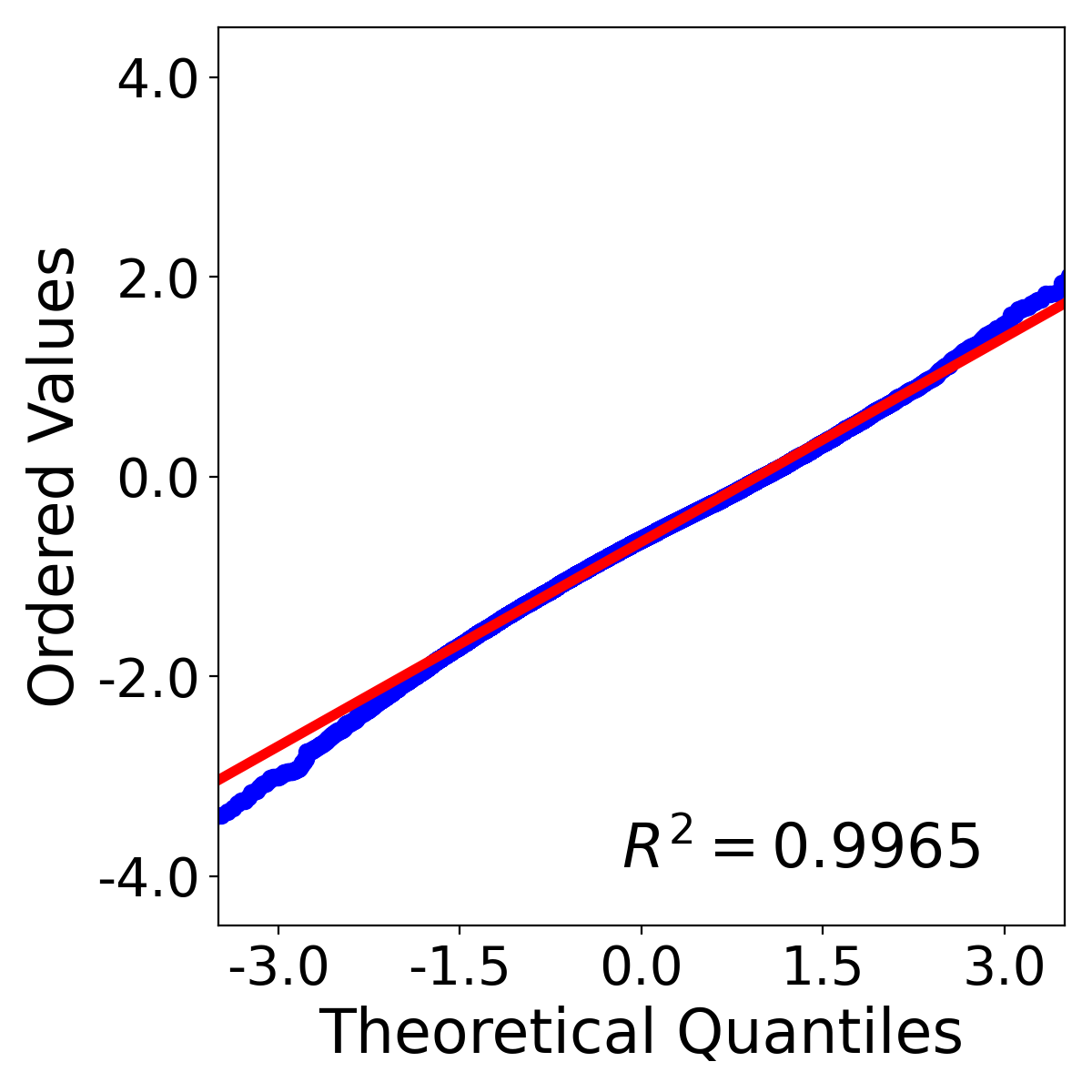}
    \end{subfigure}
    \begin{subfigure}[b]{0.32\linewidth}
      \includegraphics[width=\linewidth]{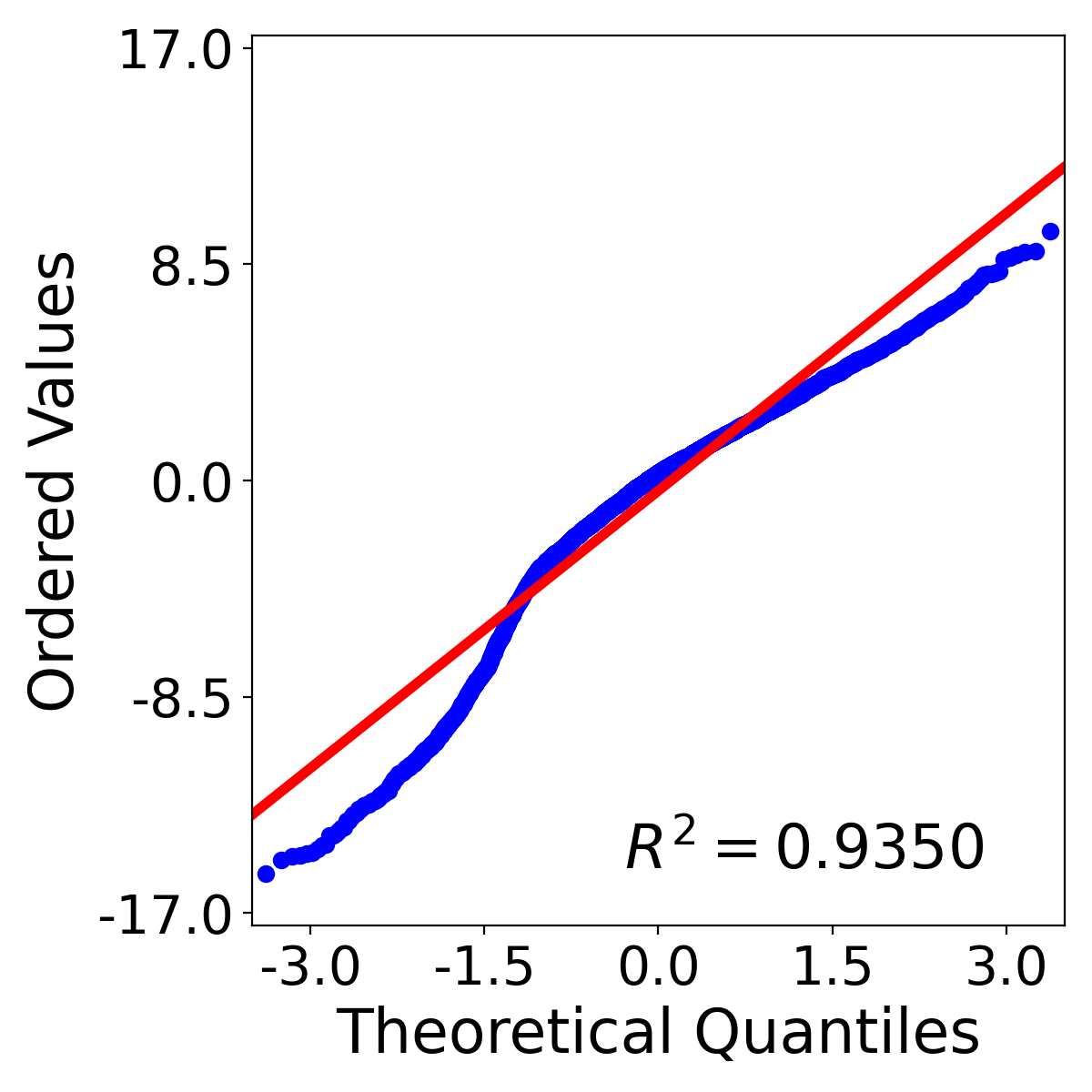}
    \end{subfigure}

    \begin{subfigure}[b]{0.32\linewidth}
	  \includegraphics[width=\linewidth]{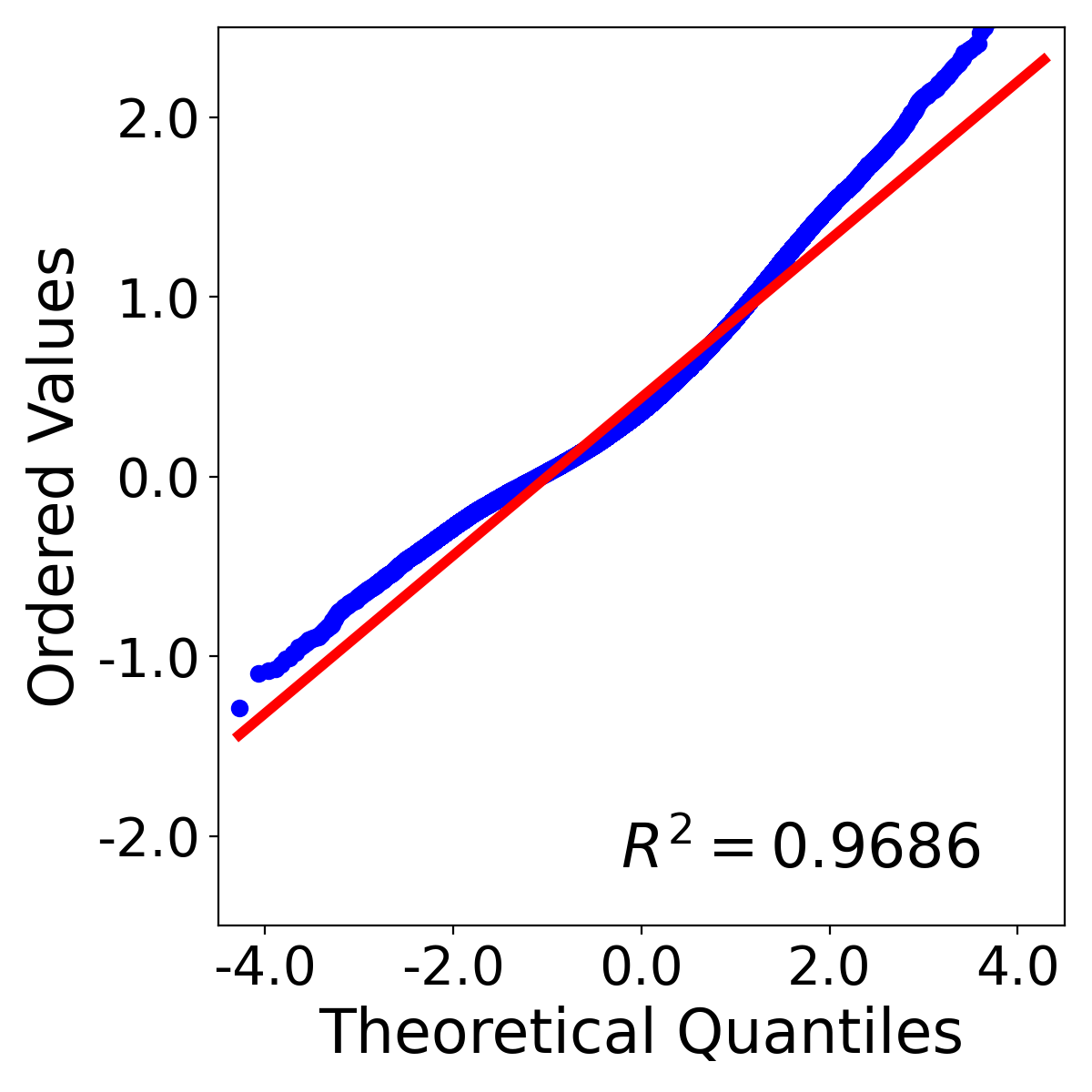}
    \end{subfigure}
    \begin{subfigure}[b]{0.32\linewidth}
      \includegraphics[width=\linewidth]{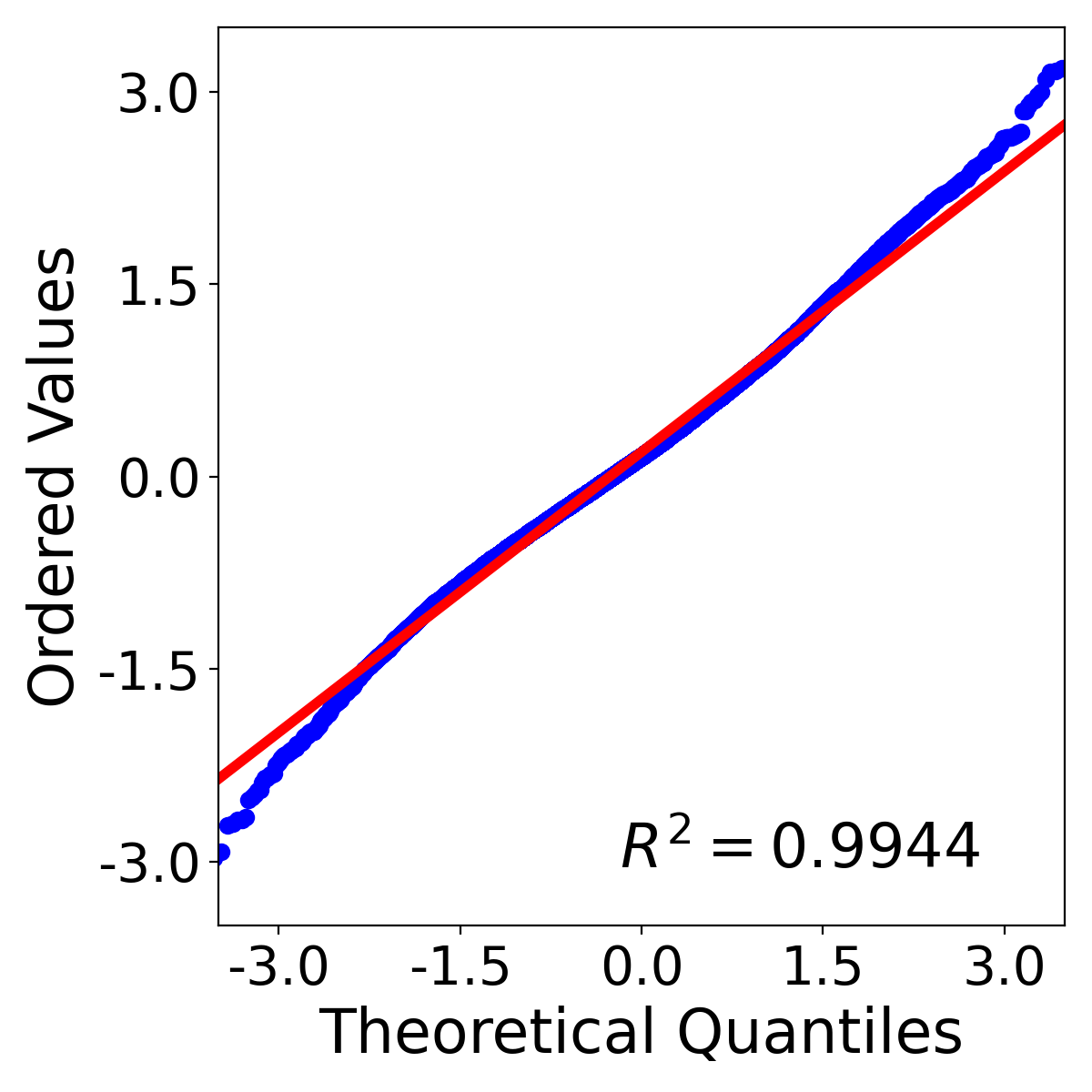}
    \end{subfigure}
    \begin{subfigure}[b]{0.32\linewidth}
      \includegraphics[width=\linewidth]{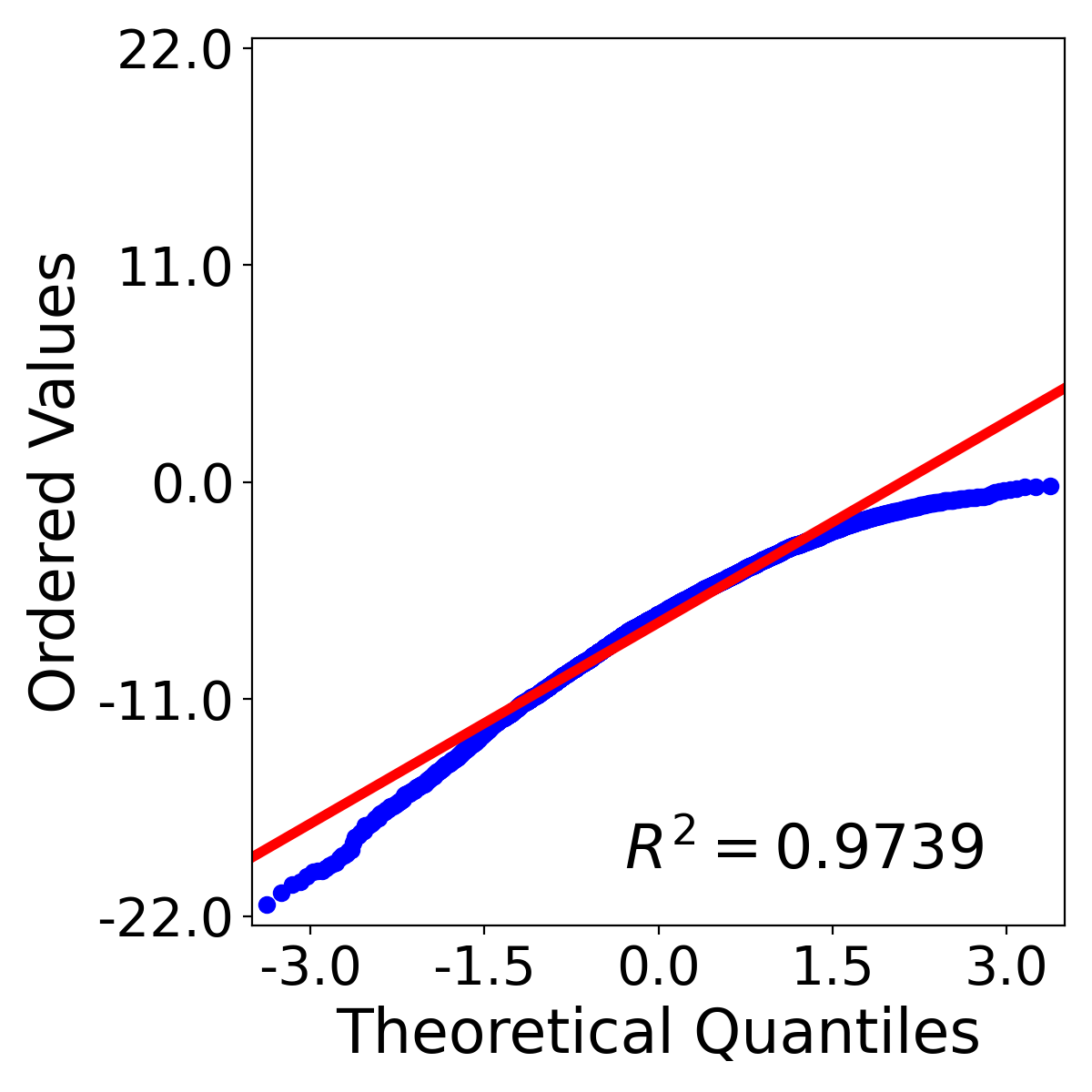}
    \end{subfigure}
    \end{minipage}
    \hfill
  \begin{subfigure}{0.42\textwidth}
  \centering
 \includegraphics[width=\linewidth]{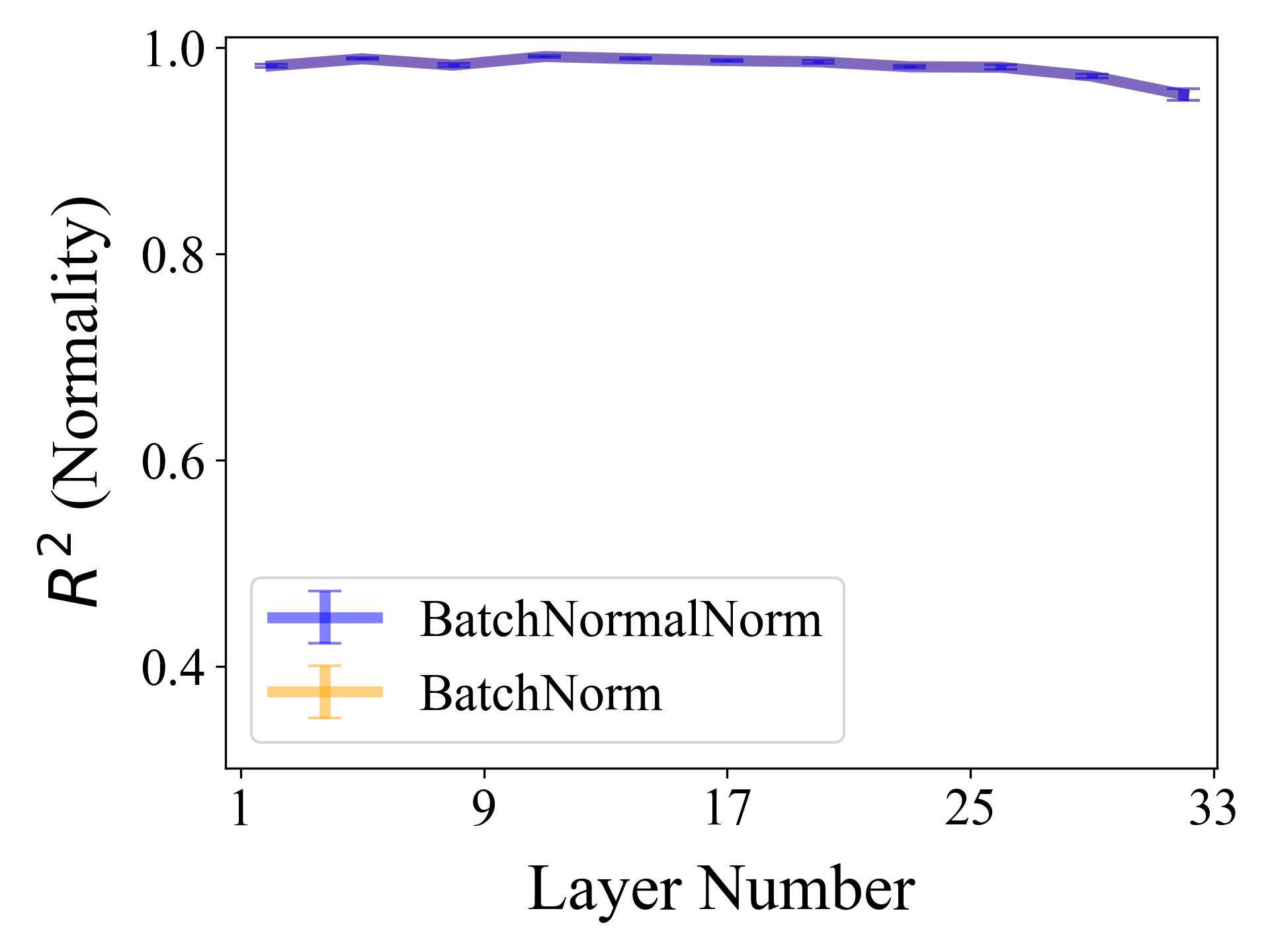}
    \end{subfigure}
    \vspace{-3.4mm}
    \caption{As in Figure \ref{figure-qq} and Figure \ref{figure-r2},
    but for networks at initialization.
    The plots
    provide a graphical illustration of the fact that
    at initialization, networks
    using
    either BatchNormalNorm or BatchNorm have close-to Gaussian pre-activations. However, as the networks are trained, BatchNormalNorm enforces and
    maintains
    normality while BatchNorm does not, as evidenced by Figure \ref{figure-r2}.
    }
    \label{figure-initialization-qq-r2}
\end{figure*}

\pdfbookmark[2]{Normality at Initialization}{bookmark-appendix-normality-initialization}
\subsection{Normality at Initialization}
\label{appendix-normality-initialization}
Figure \ref{figure-initialization-qq-r2} shows representative
Q--Q
plots,
together with an aggregate measure of normality across model layers, for post-power transform feature values when using BatchNormalNorm, and post-normalization values when using BatchNorm,
for models at initialization.
It
provides a graphical illustration of the fact that
at initialization, the pre-activations are
close
to
Gaussian regardless of the normalization layer employed;
and thus
that only the model trained with BatchNormalNorm
enforces and
maintains
normality throughout training,
as evidenced by Figure \ref{figure-r2}.
Note that the Q--Q plots presented in Figure \ref{figure-qq}
and Figure \ref{figure-initialization-qq-r2}
are obtained for the same corresponding minibatch and channel combinations.

\pdfbookmark[2]{Joint Normality and Independence Between Features}{bookmark-appendix-joint-normality-independence}
\subsection{Joint Normality and Independence Between Features}
\label{appendix-joint-normality-independence}
Following the motivation we presented in Subsection \ref{text-maximally-independent},
here we
explore the potential effect normality normalization may have on increasing joint normality in the features, and the extent to which it may increase the independence between features.

We
use the
following experimental setup.
For each layer of a
ResNet34/STL10 model
trained to convergence
using either BatchNormalNorm or BatchNorm,
we compute
the
correlation,
joint normality,
and mutual information
over
$10$
pairs of channels,
and
across $10$
validation minibatches.

We evaluate joint normality
using the negative of the HZ-statistic \citep{doi:10.1080/03610929008830400} (higher values indicate greater joint normality),
and evaluate independence using the
adjusted mutual information (AMI) metric\footnote{The AMI is a variation of mutual information, which adjusts for random chance. It is also bounded between $0$ and $1$, which makes it easier to interpret.}
\citep{JMLR:v11:vinh10a}
(lower values indicate a greater degree of independence).

We evaluate
joint normality
across pairs of channels rather than across all of the channels in a layer, because
measures of joint normality are
sensitive to
small deviations
in sample statistics
for finite sample sizes
\citep{doi:10.1080/02664763.2013.839637,EbnerHenze2020_1000129679}.
Wherever we measure AMI, we use the square root of the number of sampled features as the number of bins (a generally accepted rule of thumb) when discretizing the features, and we use uniform binning, which is appropriate for (close to) normally distributed data.

Figure \ref{figure-corr-joint-mi} demonstrates that models trained with BatchNormalNorm have higher joint normality, and have greater independence,
across the model's layers.
This is of value in context of the benefit
feature independence is thought to provide, which was explored in
Motivation
Subsection \ref{text-maximally-independent}.

\begin{figure}[h!]
  \centering
  \begin{minipage}[t]{0.99\textwidth}
  \centering
    \begin{subfigure}[b]{0.32\linewidth}
	  \includegraphics[width=\linewidth]{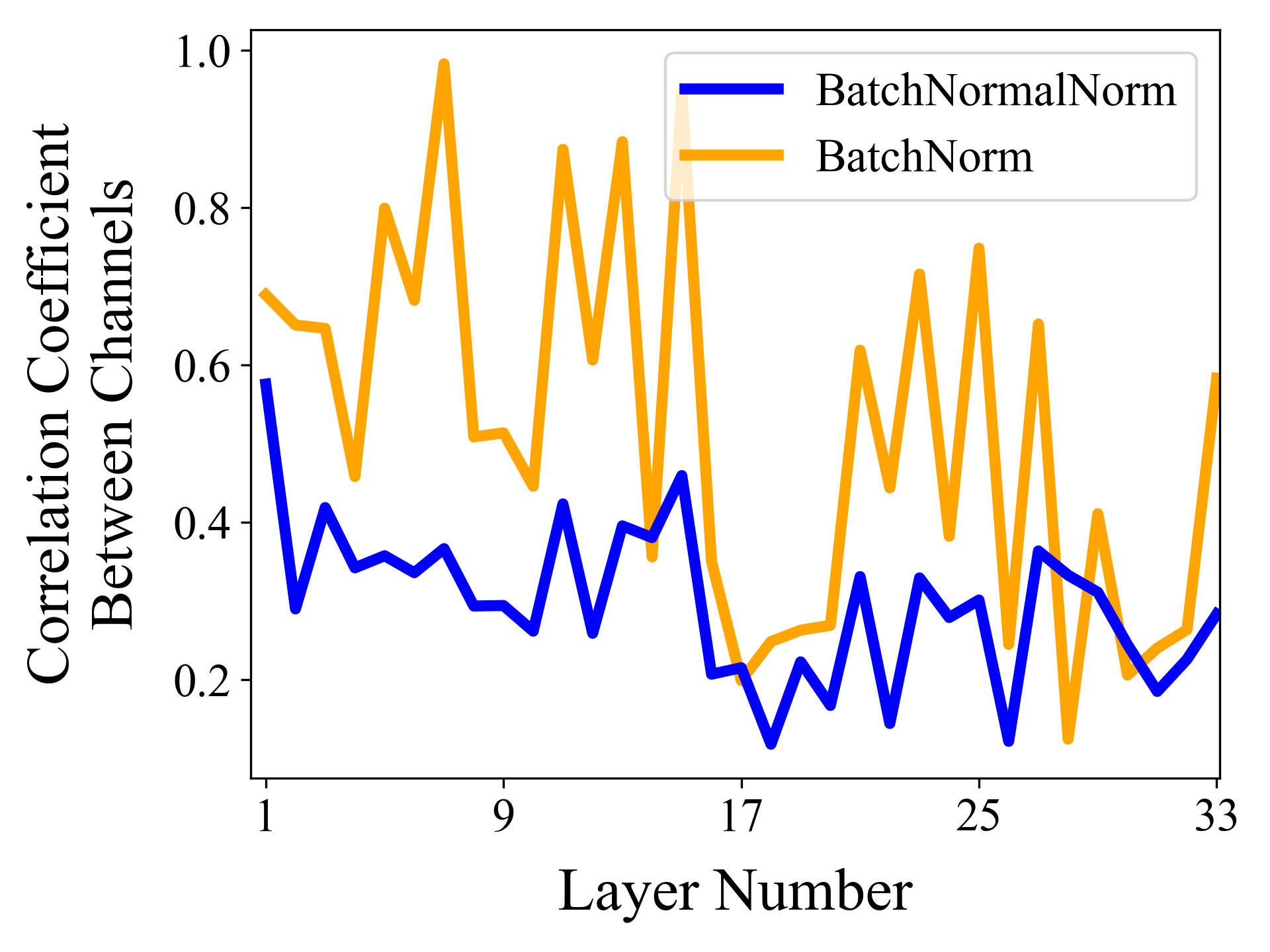}
    \end{subfigure}
    \begin{subfigure}[b]{0.32\linewidth}
      \includegraphics[width=\linewidth]{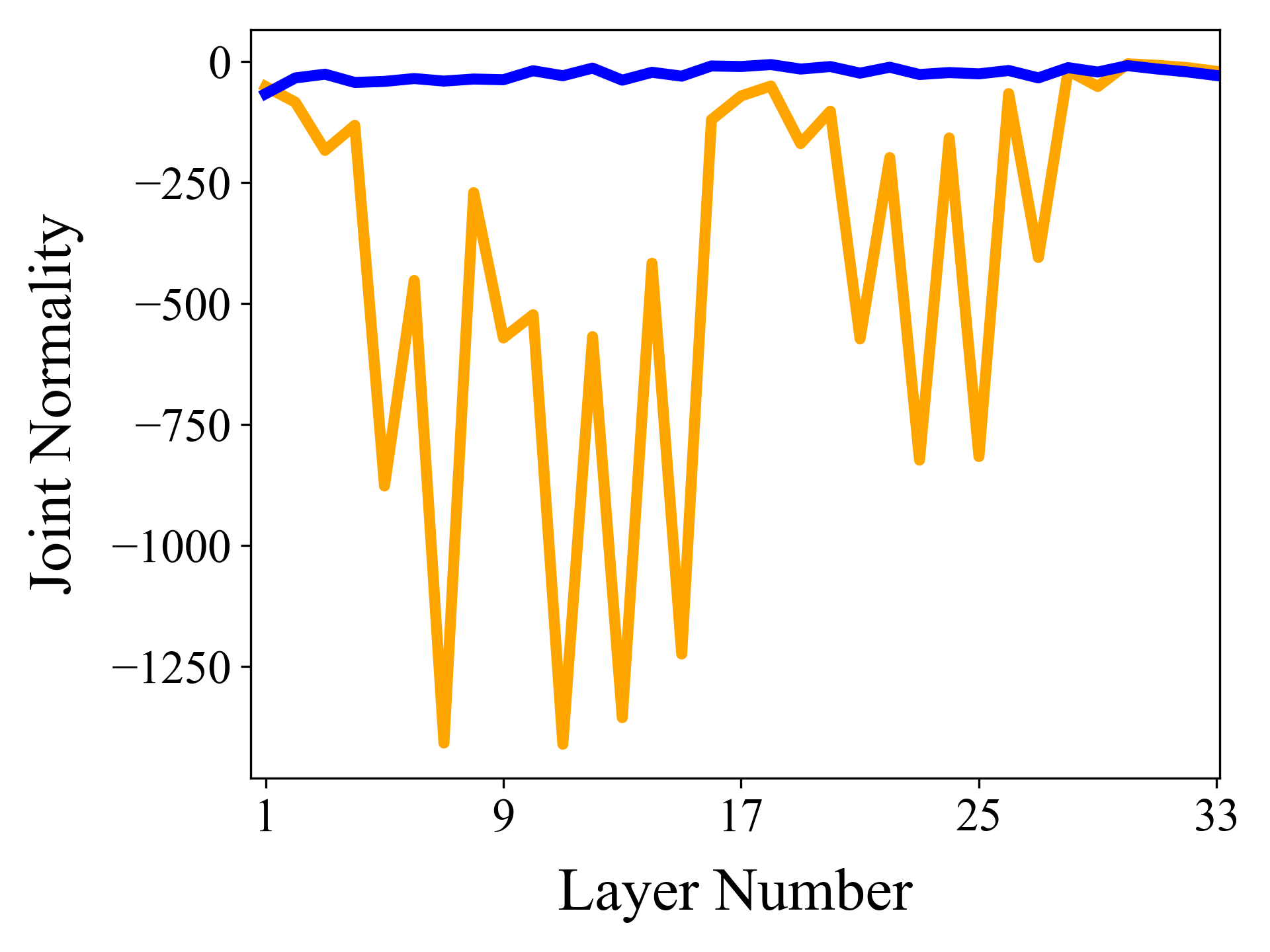}
    \end{subfigure}
    \begin{subfigure}[b]{0.32\linewidth}
      \includegraphics[width=\linewidth]{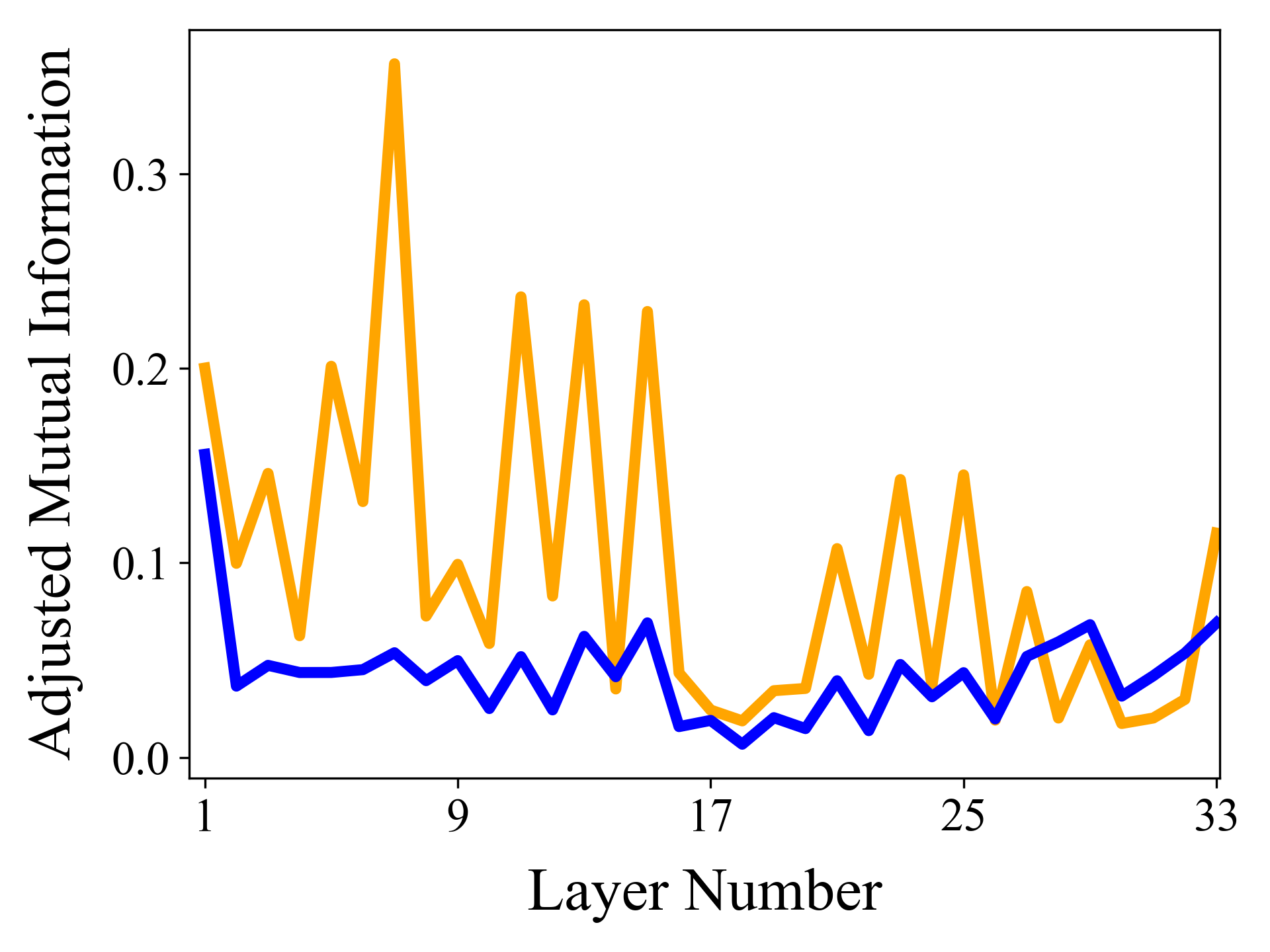}
    \end{subfigure}
  \end{minipage}
    \vspace{-3.4mm}
    \caption{\textbf{Normality normalization induces greater feature independence.}
    Correlation, joint normality, and adjusted mutual information between pairs of channels
    for models trained to convergence
    using
    BatchNormalNorm vs. BatchNorm (ResNet34/STL10).
    The results
    are obtained by averaging the corresponding statistics across $10$ channel pairs, and across $10$ validation minibatches.
    Here joint normality is quantified using the negative of the HZ-statistic.
    }
    \label{figure-corr-joint-mi}
\end{figure}

\pdfbookmark[1]{Lemmas}{bookmark-minimal-mi-given-marginals}
\section{Lemmas}
\label{minimal-mi-given-marginals}
\begin{lemma}
\label{lemma-minimal-mi-given-marginals}
Bivariate Normality Minimizes Mutual Information.
Let $X_{1} \sim \mathcal{N}\left(x_{1}; \mu_{1}, \sigma_{1}^{2}\right)$ and $X_{2} \sim \mathcal{N}\left(x_{2}; \mu_{2}, \sigma_{2}^{2}\right)$.
Their mutual information $I\left(X_{1};X_{2}\right)$ is minimized when the random variables are furthermore jointly normally distributed, i.e. $\left(X_{1}, X_{2}\right)\sim \mathcal{N}\left(\bm{x}; \bm{\mu}, \bm{\Sigma}\right)$, with $\bm{x} = \begin{bmatrix}x_{1}\\x_{2}\end{bmatrix}$, $\bm{\mu}=\begin{bmatrix}\mu_{1}\\\mu_{2}\end{bmatrix}$, $\bm{\Sigma} = \begin{bmatrix} \sigma_{1}^{2} & \rho\sigma_{1}\sigma_{2}\\\rho\sigma_{1}\sigma_{2} & \sigma_{2}^{2} \end{bmatrix}$, and $\rho$ the correlation coefficient between $X_{1}, X_{2}$.
\end{lemma}

\begin{proof}
Consider two possible distributions, \(f,g\), for the joint distribution over \(\left(X_{1},X_{2}\right)\), where \(f\) denotes the probability density function (PDF) of the bivariate normal distribution, and \(g\) can be any joint distribution. Our goal is to show that the mutual information between \(X_{1},X_{2}\), when they are distributed according to \(g\), is lower-bounded by the mutual information between \(X_{1},X_{2}\) when they are distributed according to \(f\).

For clarity of presentation, let the number of variables \(f\) and \(g\) take as arguments be clear from context, so that it is understood when they are used to denote their marginal distributions. Furthermore let \(I_{g}\left(X_{1};X_{2}\right)\) represent the mutual information when \(\left(X_{1}, X_{2}\right)\) are distributed according to \(g\), with the notation extending analogously to their joint \(h_{g}\left(X_{1},X_{2}\right)\) and marginal \(h_{g}\left(X_{1}\right)\), \(h_{g}\left(X_{2}\right)\) entropies under \(g\).

We then have
\begin{equation}
\begin{aligned}
	I_{g}\left(X_{1};X_{2}\right)
	&= h_{g}\left(X_{1}\right) + h_{g}\left(X_{2}\right) - h_{g}\left(X_{1},X_{2}\right)\\
	&= h_{f}\left(X_{1}\right) + h_{f}\left(X_{2}\right) - h_{g}\left(X_{1},X_{2}\right)\\
	&\ge h_{f}\left(X_{1}\right) + h_{f}\left(X_{2}\right) - h_{f}\left(X_{1},X_{2}\right)\\
	&= I_{f}\left(X_{1};X_{2}\right)\\
	&= \frac{1}{2}\log\left(2\pi e\sigma_{1}^{2}\right) + \frac{1}{2}\log\left(2\pi e\sigma_{2}^{2}\right) - \frac{1}{2}\log\left(\left(2\pi e\right)^{2}\left(1-\rho^{2}\right)\sigma_{1}^{2}\sigma_{2}^{2}\right)\\
	&= \frac{1}{2}\log\left(\frac{1}{1-\rho^{2}}\right),
\end{aligned}
\end{equation}
where the second equality follows because by assumption the marginals are normally distributed, the inequality follows because the normal distribution maximizes entropy,
and in the second-last equality we have used the expressions for
the entropies of the univariate and bivariate normal distributions.
\label{lemma-bivariate-normality}
\end{proof}

Consequently,
when the random variables are jointly normally distributed, \(\rho=0\) implies \(I\left(X_{1};X_{2}\right) = 0\); thus uncorrelatedness implies independence.\footnote{The preceding result extends straightforwardly to the general multivariate setting, i.e. with more than two random variables.}

\pdfbookmark[1]{Derivation of the Power Transform Negative Log-Likelihood}{bookmark-appendix-nll-derivation}
\section{Derivation of the Power Transform Negative Log-Likelihood}
\label{appendix-nll-derivation}
As in Section \ref{background},
consider a random variable $H$
from which a sample
$\bm{h} = \left\{h_{i}\right\}_{i=1}^{N}$ is obtained. Recall that the power transform gaussianizes $\bm{h}$ by applying Equation \ref{transform} for each $h_{i}$,
where the parameter $\lambda$ is obtained using maximum likelihood estimation,
so that the transformed variable is as normally distributed as possible.

Denote the transformed random variable as $X \sim \mathcal{N}\left(x; \mu,\sigma^{2}\right)$, where $x = \psi\left(h; \lambda\right)$, and $\mu = \mu\left(\lambda\right)$, $\sigma^{2} =\sigma^{2}\left(\lambda\right)$; i.e. $\mu$ and $\sigma^{2}$ are in general functions of $\lambda$.
Obtaining the
maximum likelihood estimates
for $\left(\mu, \sigma^{2}, \lambda\right)$, shown next, requires evaluating the probability density function (PDF) of $H$, given the transformed variable $X$ is normally distributed.

Evaluating the cumulative distribution function (CDF) of $H$ gives\footnote{To simplify the presentation of the derivation,
we omit the cases where \(\lambda=0\) and \(\lambda=2\),
and outline the NLL for \(h \ge 0\) only, as the case for \(h < 0\) follows closely by symmetry.}
\begin{equation}
\begin{aligned}
	F_{H}\left(h\right)
	&= P\left(H \le h\right)\\
	&= P\left(\left(1+\lambda X\right)^{1/\lambda}-1 \le h\right)\\
	&= P\left(X \le \frac{1}{\lambda}\left(\left(1+h\right)^{\lambda} - 1\right)\right)\\
	&= F_{X}\left(\frac{1}{\lambda}\left(\left(1+h\right)^{\lambda} - 1\right)\right).
\end{aligned}
\label{eqn-cdf}
\end{equation}

Differentiating Equation \ref{eqn-cdf} gives the PDF of $H$:
\begin{equation}
\begin{aligned}
	f_{H}\left(h;\mu\left(\lambda\right),\sigma^{2}\left(\lambda\right),\lambda\right)
	&= \frac{d}{dh}F_{H}\left(h\right)\\
	&= \frac{d}{dh}\left(\frac{1}{2}\left(1 + \text{erf}\left(\frac{\frac{1}{\lambda}\left(\left(1+h\right)^{\lambda} - 1\right)-\mu\left(\lambda\right)}{\sigma\left(\lambda\right)\sqrt{2}}\right)\right)\right)\\
	&=\left(1+h\right)^{\lambda-1}\frac{1}{\sqrt{2\pi\sigma^2\left(\lambda\right)}}\exp\left(\frac{-1}{2\sigma^2\left(\lambda\right)}\left(\frac{\left(1+h\right)^{\lambda}-1}{\lambda}-\mu\left(\lambda\right)\right)^{2}\right)\\
	&= \left(1+h\right)^{\lambda-1} f_{X}\left(x;\mu\left(\lambda\right), \sigma^2\left(\lambda\right)\right).
\end{aligned}
\end{equation}

The negative log-likelihood (NLL) of the sample according to this distribution is
\begin{equation}
\begin{aligned}
	&\mathcal{L}\left(\bm{h};\mu\left(\lambda\right),\sigma^{2}\left(\lambda\right),\lambda\right)\\
	&= -\frac{1}{N}\log\prod_{i=1}^{N}f_{H}\left(h_{i};\mu\left(\lambda\right),\sigma^{2}\left(\lambda\right),\lambda\right)\\
	&= -\frac{1}{N}\log\prod_{i=1}^{N}\left[\left(1+h_{i}\right)^{\lambda-1}\frac{1}{\sqrt{2\pi\sigma^2\left(\lambda\right)}}\exp\left(\frac{-1}{2\sigma^2\left(\lambda\right)}\left(x_{i}-\mu\left(\lambda\right)\right)^{2}\right)\right]\\
	 &= \frac{1}{2}\log\left(2\pi\right) + \frac{1}{2}\log\left(\sigma^2\left(\lambda\right)\right)
	+ \frac{1}{2N\sigma^2\left(\lambda\right)}\sum_{i=1}^{N}\left(x_{i}-\mu\left(\lambda\right)\right)^{2}
	- \frac{\lambda-1}{N} \sum_{i=1}^{N}\log\left(1+h_{i}\right).
\end{aligned}
\end{equation}

Optimizing the NLL w.r.t. $\mu\left(\lambda\right)$ and $\sigma^{2}\left(\lambda\right)$ gives
\begin{equation}
\begin{aligned}
	\hat{\mu}\left(\lambda\right)&=\frac{1}{N}\sum_{i=1}^{N}x_{i},\\
	\hat{\sigma}^{2}\left(\lambda\right)&=\frac{1}{N}\sum_{i=1}^{N}\left(x_{i}-\hat{\mu}\left(\lambda\right)\right)^{2}.
\end{aligned}
\end{equation}

Finally re-writing the NLL using the expressions for $\hat{\mu}\left(\lambda\right)$ and $\hat{\sigma}^{2}\left(\lambda\right)$, we obtain the profile NLL \citep{pickles1985introduction}:
\begin{equation}
\begin{aligned}
	\mathcal{L}\left(\bm{h};\hat{\mu}\left(\lambda\right),\hat{\sigma}^{2}\left(\lambda\right),\lambda\right)
	&=\frac{1}{2}\left(\log\left(2\pi\right)+1\right) + \frac{1}{2}\log\left(\hat{\sigma}^2\left(\lambda\right)\right)- \frac{\lambda-1}{N} \sum_{i=1}^{N}\log\left(1+h_{i}\right).
\end{aligned}
\end{equation}

\pdfbookmark[1]{Series Expansion of the Power Transform Loss}{bookmark-series-expansion-loss}
\section{Series Expansion of the Power Transform Loss}
\label{series-expansion-loss}
Let
$\mathcal{L}_{2}\left(\bm{h};\left(\lambda,\lambda_{0}=1\right)\right)$
denote the second-order series expansion of the power transform's
negative log-likelihood
(NLL)
centered at $\lambda_{0} = 1$, i.e.
\begin{equation}
	\begin{aligned}
		\mathcal{L}_{2}\left(\bm{h};\left(\lambda,\lambda_{0}=1\right)\right)
		&= \mathcal{L}\left(\bm{h};\lambda=1\right)
		+ \left(\lambda - 1\right)\mathcal{L}'\!\left(\bm{h};\lambda=1\right)
		+ \frac{\left(\lambda - 1\right)^{2}}{2}\mathcal{L}''\!\left(\bm{h};\lambda=1\right).
	\end{aligned}
\label{loss-series}
\end{equation}

We have\footnote{To simplify the presentation, we outline the series expansion for $h \ge 0$ only, as
$h < 0$ follows closely by symmetry.}
\begin{equation}
	\begin{aligned}
		\mathcal{L}\left(\bm{h};\lambda=1\right)
		&= \mathcal{L}\left(\bm{h};\lambda\right)\Bigr|_{\lambda=1}\\
		&= \frac{1}{2}\log\left(2\pi + 1\right) + \frac{1}{2}\log\left(\hat{\sigma}^{2}\left(\lambda=1\right)\right),\\
		\mathcal{L}'\!\left(\bm{h};\lambda=1\right)
		&= \frac{\partial\mathcal{L}\left(\bm{h};\lambda\right)}{\partial\lambda}\Bigr|_{\lambda=1}\\
		&= \frac{1}{2\hat{\sigma}^{2}\left(\lambda=1\right)}\frac{\partial\hat{\sigma}^{2}\left(\lambda\right)}{\partial\lambda}\Bigr|_{\lambda=1} - \frac{1}{N}\sum_{i=1}^{N}\log\left(1+h_{i}\right),\\
		\mathcal{L}''\!\left(\bm{h};\lambda=1\right)
		&=
		\frac{\partial^{2}\mathcal{L}\left(\bm{h};\lambda\right)}{\partial\lambda^{2}}\Bigr|_{\lambda=1}\\
		&= \frac{-1}{2\left(\hat{\sigma}^{2}\left(\lambda=1\right)\right)^{2}}\left(\frac{\partial\hat{\sigma}^{2}\left(\lambda\right)}{\partial\lambda}\Bigr|_{\lambda=1}\right)^{2}
		+ \frac{1}{2\hat{\sigma}^{2}\left(\lambda=1\right)}\frac{\partial^{2}\hat{\sigma}^{2}\left(\lambda\right)}{\partial\lambda^{2}}\Bigr|_{\lambda=1},
	\end{aligned}
\label{loss-series012-lambda=1}
\end{equation}
where
\begin{equation}
	\begin{aligned}
		\frac{\partial\hat{\sigma}^{2}\left(\lambda\right)}{\partial\lambda}
		&= \frac{2}{N}\sum_{i=1}^{N}\left[\left(\psi\left(h_{i};\lambda\right) - \hat{\mu}\left(\lambda\right)\right) \left(\frac{\partial\psi\left(h_{i};\lambda\right)}{\partial\lambda} - \frac{\partial\hat{\mu}\left(\lambda\right)}{\partial\lambda}\right)\right],
	\end{aligned}
\end{equation}
\begin{equation}
	\begin{aligned}
		\therefore \frac{\partial\hat{\sigma}^{2}\left(\lambda\right)}{\partial\lambda}\Bigr|_{\lambda=1}
		&= \frac{2}{N}\sum_{i=1}^{N}\left[\left(h_{i}-\hat{\mu}\left(\lambda=1\right)\right)\left(\frac{\partial\psi\left(h_{i};\lambda\right)}{\partial\lambda}\Bigr|_{\lambda=1} - \frac{\partial\hat{\mu}\left(\lambda\right)}{\partial\lambda}\Bigr|_{\lambda=1}\right)\right],
	\end{aligned}
\end{equation}
with
\begin{equation}
	\begin{aligned}
		\frac{\partial\psi\left(h_{i};\lambda\right)}{\partial\lambda}\Bigr|_{\lambda=1}
		&= \left(1+h_{i}\right)\left(\log\left(1+h_{i}\right)\right)-h_{i},\\
		\frac{\partial\hat{\mu}\left(\lambda\right)}{\partial\lambda}\Bigr|_{\lambda=1}
		&= \frac{1}{N}\sum_{i=1}^{N}\frac{\partial\psi\left(h_{i};\lambda\right)}{\partial\lambda}\Bigr|_{\lambda=1},
	\end{aligned}
\label{dvariance}
\end{equation}
and
\begin{equation}
	\begin{aligned}
		\frac{\partial^{2}\hat{\sigma}^{2}\left(\lambda\right)}{\partial\lambda^{2}} =
		\frac{2}{N}\sum_{i=1}^{N}&\left[
		\left(\left(\psi\left(h_{i};\lambda\right) - \hat{\mu}\left(\lambda\right)\right) \left(\frac{\partial^{2}\psi\left(h_{i};\lambda\right)}{\partial\lambda^{2}} - \frac{\partial^{2}\hat{\mu}\left(\lambda\right)}{\partial\lambda^{2}}\right)\right)\right.\\
		&\left.+\left(\frac{\partial\psi\left(h_{i};\lambda\right)}{\partial\lambda} - \frac{\partial\hat{\mu}\left(\lambda\right)}{\partial\lambda}\right)^{2}\right],
	\end{aligned}
\label{ddvariance1}
\end{equation}
\begin{equation}
	\begin{aligned}
		\therefore \frac{\partial^{2}\hat{\sigma}^{2}\left(\lambda\right)}{\partial\lambda^{2}}\Bigr|_{\lambda=1} =
		\frac{2}{N}\sum_{i=1}^{N}
		&\left[\left(\left(h_{i} - \hat{\mu}\left(\lambda=1\right)\right) \left(\frac{\partial^{2}\psi\left(h_{i};\lambda\right)}{\partial\lambda^{2}}\Bigr|_{\lambda=1} - \frac{\partial^{2}\hat{\mu}\left(\lambda\right)}{\partial\lambda^{2}}\Bigr|_{\lambda=1}\right)\right)
		\right.\\
		&\left.+\left(\frac{\partial\psi\left(h_{i};\lambda\right)}{\partial\lambda}\Bigr|_{\lambda=1} - \frac{\partial\hat{\mu}\left(\lambda\right)}{\partial\lambda}\Bigr|_{\lambda=1}\right)^{2}\right],
	\end{aligned}
\end{equation}
with
\begin{equation}
	\begin{aligned}
		\frac{\partial^{2}\psi\left(h_{i};\lambda\right)}{\partial\lambda^{2}}\Bigr|_{\lambda=1} &= \left(1+h_{i}\right)\left(\log\left(1+h_{i}\right)\right)^{2} - 2\frac{\partial\psi\left(h_{i};\lambda\right)}{\partial\lambda}\Bigr|_{\lambda=1},\\
		\frac{\partial^{2}\hat{\mu}\left(\lambda\right)}{\partial\lambda^{2}}\Bigr|_{\lambda=1} &= \frac{1}{N}\sum_{i=1}^{N}\frac{\partial^{2}\psi\left(h_{i};\lambda\right)}{\partial\lambda^{2}}\Bigr|_{\lambda=1}.
	\end{aligned}
\label{ddvariance2}
\end{equation}
Furthermore, because the power transform is applied after the normalization step (see main text), \(\hat{\mu}\left(\lambda=1\right) = 0\) and \(\hat{\sigma}^{2}\left(\lambda=1\right) = 1\).

\pdfbookmark[1]{Evaluation of \texorpdfstring{$\hat{\lambda}$}{λ̂} Estimates}{bookmark-estimation-lambda}
\section{Evaluation of $\hat{\lambda}$ Estimates}
\label{estimation-lambda}
Figure \ref{nll-plots} provides representative examples substantiating the similarity between the
negative log-likelihood
(NLL)
and its second-order series expansion
around \(\lambda_{0}=1\).
The figure furthermore demonstrates the accuracy
of obtaining
estimates
of
\(\hat{\lambda}\) using one step of the Newton-Raphson method.

\begin{figure}[h!]
  \centering
  \begin{minipage}[t]{0.60\textwidth}
    \begin{subfigure}[b]{0.32\linewidth}
	  \includegraphics[width=\linewidth]{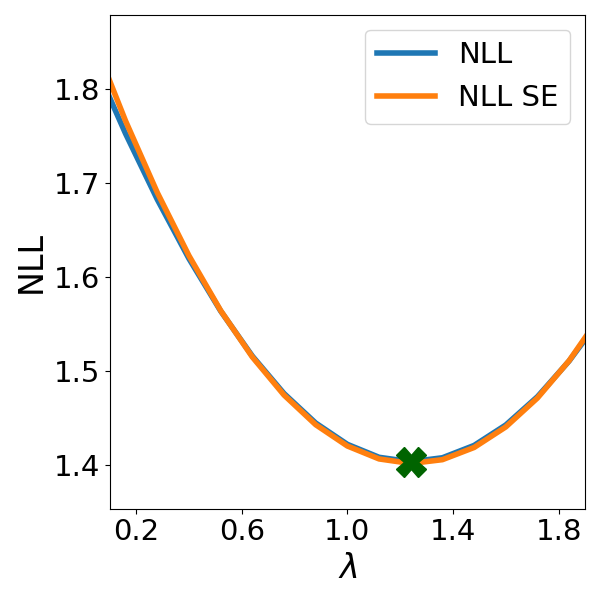}
    \end{subfigure}
    \begin{subfigure}[b]{0.32\linewidth}
      \includegraphics[width=\linewidth]{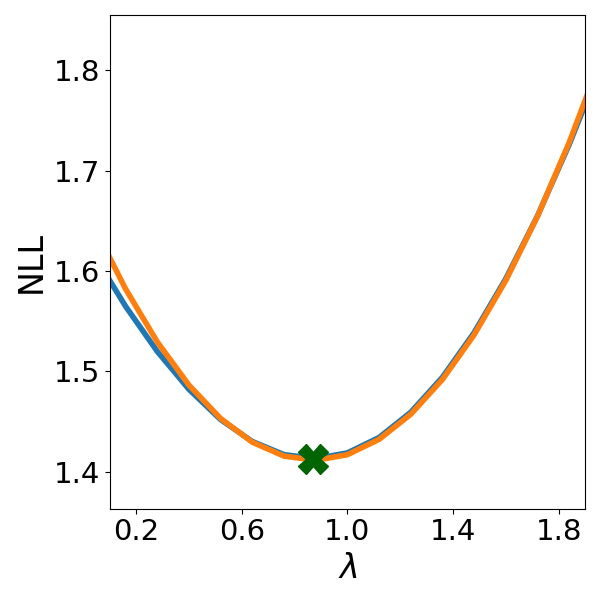}
    \end{subfigure}
    \begin{subfigure}[b]{0.32\linewidth}
      \includegraphics[width=\linewidth]{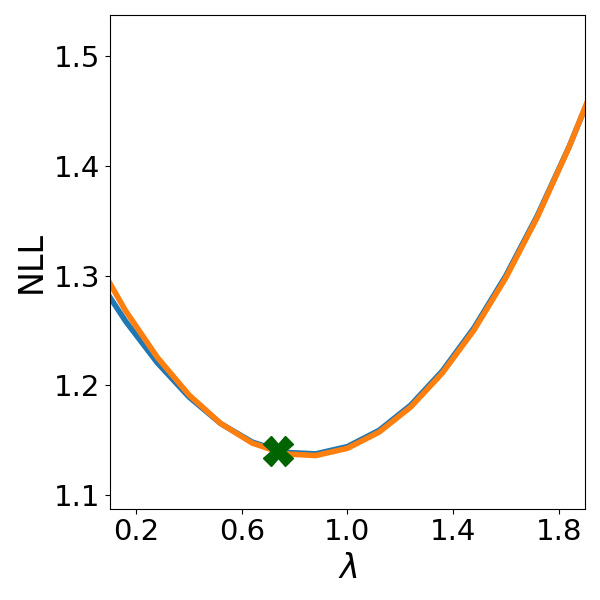}
    \end{subfigure}
  \end{minipage}
  \begin{minipage}[t]{0.60\textwidth}
    \begin{subfigure}[b]{0.32\linewidth}
      \includegraphics[width=\linewidth]{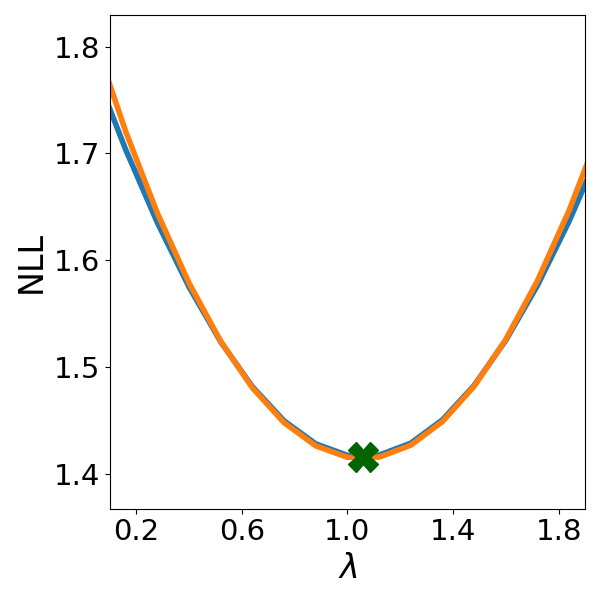}
    \end{subfigure}
    \begin{subfigure}[b]{0.32\linewidth}
      \includegraphics[width=\linewidth]{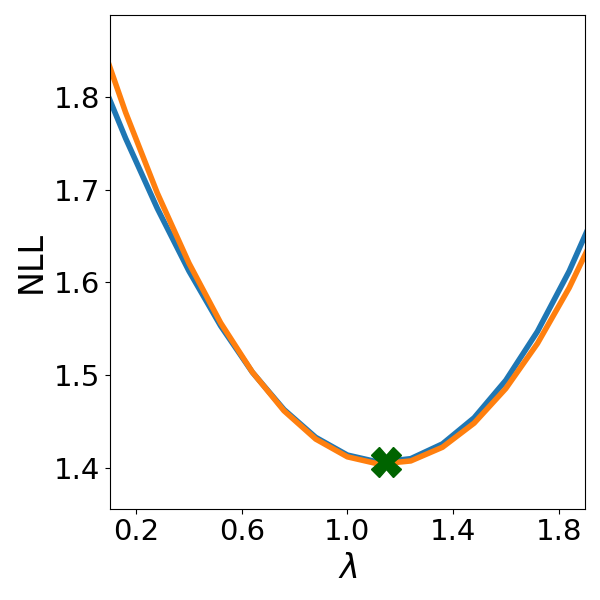}
    \end{subfigure}
    \begin{subfigure}[b]{0.32\linewidth}
      \includegraphics[width=\linewidth]{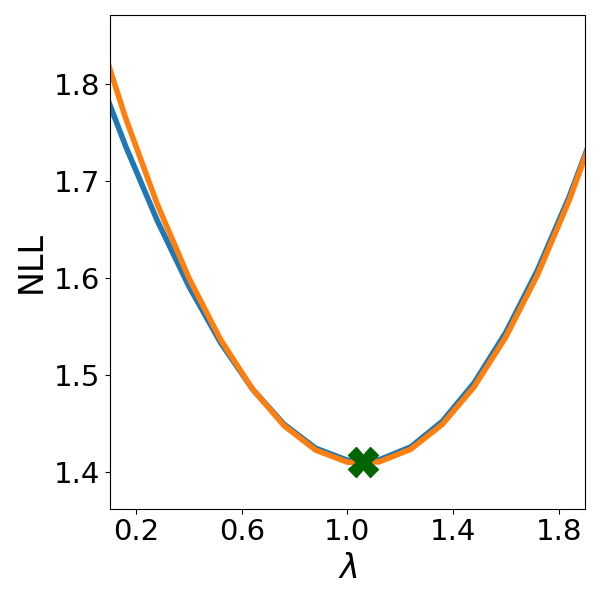}
    \end{subfigure}
  \end{minipage}
  \begin{minipage}[t]{0.60\textwidth}
    \begin{subfigure}[b]{0.32\linewidth}
      \includegraphics[width=\linewidth]{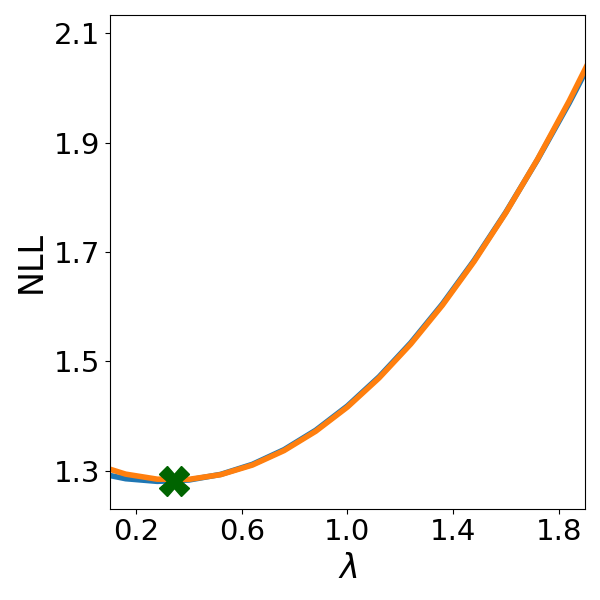}
    \end{subfigure}
    \begin{subfigure}[b]{0.32\linewidth}
      \includegraphics[width=\linewidth]{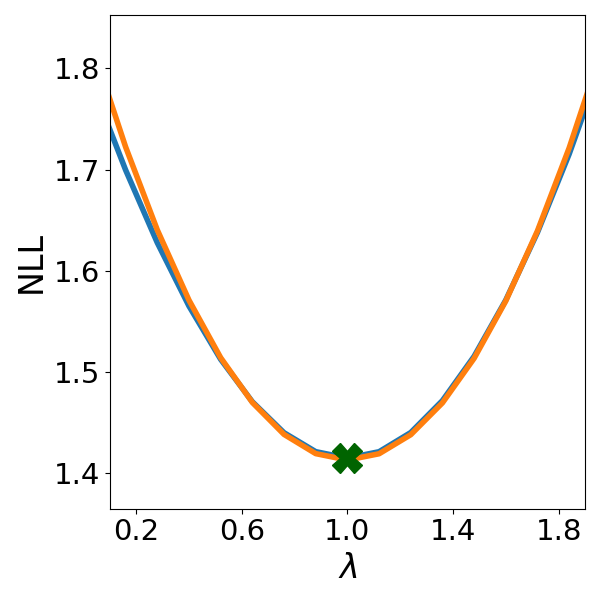}
    \end{subfigure}
    \begin{subfigure}[b]{0.32\linewidth}
      \includegraphics[width=\linewidth]{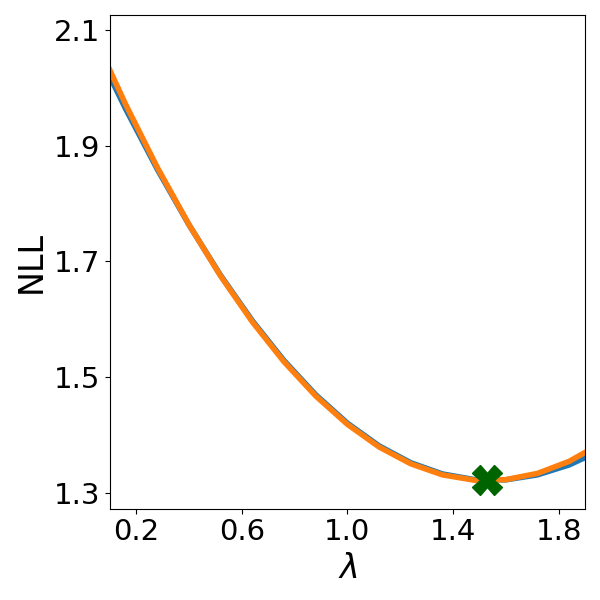}
    \end{subfigure}
    \vspace{-2.5mm}
    \caption{Normality normalization
    estimates for $\hat{\lambda}$ for a given training minibatch (ResNet18/CIFAR10). Left to right: increasing layer number. Top to bottom: estimates from various channels. Normality normalization's quadratic series expansion for the loss (NLL SE) closely approximates the original loss (NLL), leading to accurate estimates of $\hat{\lambda}$ (marked by $\times$).}
    \label{nll-plots}
  \end{minipage}
\end{figure}

\pdfbookmark[1]{Training Details}{bookmark-appendix-training-details}
\section{Training Details}
\label{appendix-training-details}

\pdfbookmark[2]{ResNet and WideResNet Experiments}{bookmark-appendix-resnet-experiments}
\subsection{ResNet and WideResNet Experiments}
\label{appendix-resnet-experiments}
The training configuration of the model and dataset combinations
which use batch normality normalization (BatchNormalNorm/BNN), batch normalization (BatchNorm/BN),
instance normality normalization (InstanceNormalNorm/INN), instance normalization (InstanceNorm/IN),
group normality normalization (GroupNormalNorm/GNN), group normalization (GroupNorm/GN), decorrelated batch normality normalization (DBNN), and decorrelated batch normalization (DBN),
are as follows.

We used a variety of residual network (ResNet) \citep{7780459}
and wide residual network (WideResNet) \citep{Zagoruyko2016WRN} architectures in our experiments.
For all
of the
experiments except
those
using
the
TinyImageNet (TinyIN),
Caltech101,
and
Food101
datasets,
models were trained from random initialization
for \(200\) epochs,
with a
factor of \(10\) reduction in learning rate at each \(60\)-epoch interval.
For the experiments using the TinyImageNet (TinyIN), Caltech101, and Food101 datasets,
models were trained from random initialization for \(100\) epochs,
with a
factor of \(10\) reduction in learning rate at epochs \(40, 70, 90\).
A group size of $32$ was used in all of the relevant group normalization experiments.
For the Caltech101 dataset, each run used a random $90$/$10$\% allocation to obtain the training and validation splits respectively.\footnote{The official Caltech101 dataset does not come with
its own
training/validation split.}
Each such run used its own unique random seed to
generate
the splits for that run,
which facilitates
greater precision in the reporting of our aggregate results across the
runs.

In all of our experiments involving the ResNet18, ResNet34, and WideResNet architectures, stochastic gradient descent (SGD) with learning rate \(0.1\), weight decay
\(5 \times 10^{-4}\),
momentum \(0.9\),
and minibatch size $128$
was used.
In the experiments involving the ResNet50 architecture
on the Caltech101 and Food101 datasets,
SGD with learning rate $0.0125$, weight decay
$1 \times 10^{-4}$,
momentum $0.9$,
and minibatch size $32$ was used.
A noise factor of $\xi=0.4$,
was used,
as preliminary
experiments
demonstrated
increases typically
resulted in
training
instability.
We also investigated several hyperparameter configurations, including for the learning rate, learning rate scheduler, weight decay, and minibatch size, across all the models
and found the
present
configurations to generally work best
across all of them.

\pdfbookmark[2]{Vision Transformer Experiments}{bookmark-appendix-vit-experiments}
\subsection{Vision Transformer Experiments}
\label{appendix-vit-experiments}
The training configuration of the model and dataset combinations
which use layer normality normalization (LayerNormalNorm/LNN) and layer normalization (LayerNorm/LN)
are as follows.
We used a vision transformer \citep{NIPS20173f5ee243, dosovitskiy2021an} model
consisting of $8$ transformer layers, $8$ attention heads, hidden dimension size of $768$, and multi-layer perceptron (MLP) dimension size of $2304$.
A patch size of $4$ was used throughout, except
in
the Food101
and ImageNet
experiments,
where it was set to $16$.

For all
of the
experiments except those using the SVHN dataset,
a learning rate warm-up strategy was employed, where the learning rate was linearly increased from a $0.1$-th fraction of its value, to the full learning rate. After the warm-up phase, a cyclic learning rate schedule based on cosine annealing with periodic restarts was employed \citep{loshchilov2017sgdr}, with $50$ iterations until the first restart, a factor of $2$ for increasing the number of epochs between two warm restarts,
and
a minimal admissible learning rate of $1 \times 10^{-6}$.
Models were trained using a minibatch size of $128$,
and for
$900$ epochs on the CIFAR10, CIFAR100, Food101 datasets, and for $200$ epochs on the ImageNet dataset.
For the ImageNet experiments, we applied weighted random sampling to sample training examples based on the training set's corresponding inverse class frequency for the data point; we found this to help across all
the model configurations used.
For the experiments involving the SVHN dataset,
models were trained from random initialization
for $200$ epochs, with a factor of $10$ reduction in learning rate at each $60$-epoch interval, and a minibatch size of $32$.

The AdamW optimizer \citep{KingBa15, loshchilov2018decoupled} with learning rate \(1 \times 10^{-3}\), weight decay \(5 \times 10^{-2}\),  \(\left(\beta_{1},\beta_{2}\right) = \left(0.9, 0.999\right)\), \(\epsilon=1 \times 10^{-8}\)
was used.
A noise factor of $\xi=1.0$
was used,
as preliminary
experiments
demonstrated
increases typically
resulted in
training
instability.
We also investigated several hyperparameter configurations,
including for the learning rate, learning rate scheduler, weight decay, and minibatch size, across all the models
and found the
present
configurations to generally work best
across all of them.

The data augmentations used for the models presented in Table
\ref{results-lnn}
are as follows.
For the models trained on the SVHN dataset, mild random translations and rotations were used. For the models trained on the CIFAR10 and CIFAR100 datasets, random cropping, random horizontal flips,
mild color jitters, and mixup \citep{zhang2018mixup} were used.
For the models trained on the Food101 dataset,
random cropping with resizing, random horizontal flips, moderate color jitters, and mixup were used.
For the models trained on the ImageNet dataset,
random cropping with resizing, random horizontal flips, and moderate color jitters were used.

\pdfbookmark[2]{Datasets and Frameworks}{bookmark-appendix-datasets-frameworks}
\subsection{Datasets and Frameworks}
\label{appendix-datasets-frameworks}
The datasets we used were
CIFAR10, CIFAR100 \citep{krizhevsky2009learning},
STL10 \citep{pmlr-v15-coates11a},
SVHN \citep{netzer2011reading},
Caltech101 \citep{liandreetoranzatoperona2022},
TinyImageNet \citep{le2015tinyimagenet},
Food101 \citep{bossard14},
and ImageNet \citep{5206848}.
We trained our models using the PyTorch \citep{paszke2019pytorch} machine learning framework.

\end{document}